\documentclass[10pt]{article}
\usepackage[accepted]{tmlr}

\usepackage{amsmath,amsfonts,bm}

\def\eqref#1{equation~\ref{#1}}

\def\1{\bm{1}}

\DeclareMathAlphabet{\mathsfit}{\encodingdefault}{\sfdefault}{m}{sl}
\SetMathAlphabet{\mathsfit}{bold}{\encodingdefault}{\sfdefault}{bx}{n}

\newcommand{\E}{\mathbb{E}}

\newcommand{\R}{\mathbb{R}}

\newcommand{\softmax}{\mathrm{softmax}}

\usepackage[hidelinks]{hyperref}
\usepackage{url}

\title{Synergistic Benefits of Joint Molecule Generation\\ and Property Prediction}

\author{\name Adam Izdebski 
    \email adam.izdebski@helmholtz-munich.de \\
    \addr Institute of AI for Health, Helmholtz Zentrum Munchen \\
    \addr Technical University of Munich, TUM School of Computation, Information and Technology \\
    \addr Faculty of Mathematics, Informatics and Mechanics, University of Warsaw
    \AND
    \name Jan Olszewski \\
    \addr Faculty of Mathematics, Informatics and Mechanics, University of Warsaw
    \AND
    \name Pankhil Gawade \\
    \addr Institute of AI for Health, Helmholtz Zentrum Munchen
    \AND
    \name Krzysztof Koras \\
    \addr Ardigen SA
    \AND
    \name Serra Korkmaz \\
    \addr Institute of AI for Health, Helmholtz Zentrum Munchen
    \AND
    \name Valentin Rauscher \\
    \addr Technical University of Munich
    \AND
    \name Jakub M. Tomczak \\ 
    \addr Eindhoven University of Technology
    \AND
    \textbf{Ewa Szczurek}
    \email ewa.szczurek@helmholtz-munich.de \\
    \addr Institute of AI for Health, Helmholtz Zentrum Munchen\\
    \addr Faculty of Mathematics, Informatics and Mechanics, University of Warsaw
}

\def\month{01}  %
\def\year{2026} %
\def\openreview{\url{https://openreview.net/forum?id=jnzCOLyGOA}} %

\usepackage[table]{xcolor}         %
\usepackage{wrapfig}
\usepackage{adjustbox}
\usepackage{pifont} 
\usepackage{colortbl}
\usepackage{siunitx}
\DeclareSIUnit{\molar}{M}
\usepackage{bbm}

\let\tmlrAND\AND
\let\AND\relax
\RequirePackage{algorithm}
\RequirePackage{algorithmic}
\let\AND\tmlrAND

\usepackage{microtype}
\usepackage{graphicx}
\usepackage{subfigure}
\usepackage{booktabs}
\usepackage{multirow}
\usepackage{rotating}
\usepackage{wrapfig}

\usepackage{amsmath}
\usepackage{amssymb}
\usepackage{mathtools}
\usepackage{amsthm}

\theoremstyle{plain}
\newtheorem{theorem}{Theorem}[section]

\newtheorem{lemma}[theorem]{Lemma}
\newtheorem{corollary}[theorem]{Corollary}
\theoremstyle{definition}

\theoremstyle{remark}

\DeclareMathOperator*{\x}{{\mathbf{x}}}

\DeclareMathOperator{\X}{\mathcal{X}}

\DeclareMathOperator{\Y}{\mathcal{Y}}
\DeclareMathOperator{\D}{\mathcal{D}}

\newcommand*{\joint}{p_{\theta}(\x, y)}
\newcommand*{\generator}{p_{\theta}(\x)}

\DeclareMathOperator*{\Keys}{\mathbf{K}}
\DeclareMathOperator*{\Queries}{\mathbf{Q}}
\DeclareMathOperator*{\Values}{\mathbf{V}}
\DeclareMathOperator*{\Mask}{\mathbf{M}}
\DeclareMathOperator*{\TaskToken}{{\small \texttt{[TASK]}}}
\DeclareMathOperator*{\MaskToken}{{\small \texttt{ATT\_Type}}}
\DeclareMathOperator*{\LMToken}{{\small \texttt{[LM]}}}
\DeclareMathOperator*{\PredToken}{{\small \texttt{[PRED]}}}
\DeclareMathOperator*{\MLMToken}{{\small \texttt{[MLM]}}}

\usepackage{mathtools}
\newcommand{\defeq}{\vcentcolon=}

\definecolor{hyformer}{RGB}{230, 190, 255}
\definecolor{hyformer-ablation}{RGB}{120, 100, 160}
\definecolor{hyformer-inline}{RGB}{40, 30, 50}
\newcommand{\hha}{\cellcolor{hyformer-ablation!10}}
\newcommand{\hh}{\cellcolor{hyformer!58.5}}
\DeclareMathOperator*{\HyF}{{\textcolor{hyformer-inline}{\textsc{Hyformer}}}}

\begin{document}

\maketitle

\begin{abstract}
Modeling the joint distribution of data samples and their properties allows to construct a single model for both data generation and property prediction, with synergistic benefits reaching beyond purely generative or predictive models. However, training joint models presents daunting architectural and optimization challenges. Here, we propose $\HyF$,~a transformer-based joint model that successfully blends the generative and predictive functionalities, using an alternating attention mechanism and a joint pre-training scheme. We show that $\HyF$ is simultaneously optimized for molecule generation and property prediction, while exhibiting synergistic benefits in conditional sampling, out-of-distribution property prediction and representation learning. Finally, we demonstrate the benefits of joint learning in a drug design use case of discovering novel antimicrobial peptides.
\end{abstract}

\makeatletter
\begingroup
\renewcommand\thefootnote{}
\footnotetext{Code available at: \url{https://github.com/szczurek-lab/hyformer}}
\endgroup
\setcounter{footnote}{0}
\makeatother

\newpage

\section{Introduction}

Developing models that simultaneously excel in both generative and predictive tasks is a long-standing challenge in machine learning \citep{bishop1994novelty, jaakkola1998exploiting, lasserre2006principled}. Joint models, which unify these tasks, offer synergistic benefits, including improved control over the generative process of the model, improved predictive robustness towards unseen, e.g., newly generated or out-of-distribution (OOD) data, and learning representations predictive of high-level molecular features \citep{nalisnick2019hybrid, grathwohl2020classifier, cao2022deep, tomczak2022deep}. These benefits are crucial for applications such as drug design, where success depends on balancing the generation of novel molecules from unexplored regions of the chemical space coupled with robust property prediction extrapolating towards the newly generated molecules~\citep{grisoniChemicalLanguageModels2022, steshin2023lohipracticalmldrug, van2025molecular}. 

However, molecule generation and property prediction are predominantly approached in separation. This division persists even though transformer-based models are state-of-the-art across both tasks \citep{bagal2022molgpt, gao2024graph, irwin2022chemformer, xia2023molebert, zhou2023unimol}. A likely reason is that joint training poses daunting challenges, as combining a generative and a predictive part into a single model may over-regularize both parts \citep{lasserre2006principled} or cause gradient interference between the generative and predictive objectives~\citep{nalisnick2019hybrid}. As a result, molecular models continue to forgo the potential benefits of joint learning. This raises a natural question, whether one can \emph{develop a transformer-based joint model optimized for both generative and predictive performance, at the same time offering the synergistic benefits of joint learning?} 

To address this challenge, we introduce $\HyF$, a joint model that combines an autoregressive transformer decoder with a bidirectional transformer encoder in a single model with shared parameters. Upon training, we alternate between using the model as a decoder and as an encoder, with either a causal or bidirectional self-attention mechanism, alleviating problems typical for joint models. We evaluate the generative and predictive performance, as well as synergistic benefits of joint learning using $\HyF$ across a variety of molecular tasks \citep{wu2018moleculenet, brown2019guacamol, steshin2023lohipracticalmldrug, chen2023ampdiffusion}. Our contributions are:
\vspace{-0.2cm}
\begin{enumerate}
    \item We propose a novel joint model, $\HyF$, that unifies the generative and the predictive task in a single set of parameters.
    \item We demonstrate the synergistic benefits of joint modeling, where $\HyF$ outperforms baselines on (i) conditional molecule generation, (ii) out-of-distribution property prediction and (iii) molecular representation learning via probing.
    \item We show that $\HyF$ rivals the generative and predictive performance of state-of-the-art purely generative and predictive models.
    \item We showcase the applicability of joint modeling in a real-world drug design use case of discovering novel antimicrobial peptides.
\end{enumerate}

\section{Related Work}\label{section:related-work} 

\paragraph{Molecule Generation}

Existing generative approaches can be categorized into sequence- and graph-based models. Sequence-based methods represent molecules as SMILES~\citep{weiningerSMILESChemicalLanguage1988} or SELFIES~\citep{Krenn_2020} and process tokenized strings using recurrent or transformer-based language models~\citep{segler2018generating, flam2022language, bagal2022molgpt}. In contrast, graph-based models treat molecules as graphs and have been implemented using variational autoencoders~\citep{liu2018constrained, jin2019junction, maziarz2022learning, hetzel2023magnet}, normalizing flows~\citep{luo2021graphdf}, energy-based models~\citep{liu2021graphebm}, and graph transformers~\citep{gao2024graph}. More recently, 3D-based generative models have been proposed to capture the spatial geometry of molecules~\citep{hoogeboom2022equivariant, guan20233d, gao2024rethinking}, however real world drug discovery pipelines continue to rely predominantly on 2D-molecular representations~\citep{xiang20243d}.

\paragraph{Molecular Property Prediction}
Analogously, prediction models leverage distinct molecular representations. Methods based on pre-trained language models predominantly work with SMILES~\citep{wang2019smiles, fabian2020molecular, irwin2022chemformer, sultan2024transformers}, while other approaches represent molecules as graphs \citep{li2021effective, wang2022molecular}. Recent methods leverage the three-dimensional spatial structure of a molecule, either using graph neural networks \citep{fang2022geometry} or transformers \citep{zhou2023unimol}. Finally, \citet{yang2019analyzing, fabian2020molecular, stokes2020deep}
incorporate pre-computed physicochemical descriptors of molecules into training. 

\paragraph{Joint Models for Molecules} Early joint models combine variational autoencoders with latent-space predictors~\citep{gomez-bombarelliAutomaticChemicalDesign2018, maziarz2022learning}. Regression Transformer~\citep{born2023regression} frames property prediction as conditional sequence generation, but lacks unconditional generative capability. Graph2Seq~\citep{gao2024graph} is a graph-based encoder-decoder transformer, trained separately as a generative or as a predictive model, but evaluated on both molecule generation and property prediction. UniGEM~\citep{feng2024unigem} is a diffusion-based model for unified generation and prediction, however specializing in 3D molecular modeling and not directly applicable to standard SMILES-based benchmarks.

Therefore, the question of whether the transformer architecture can be used to implement a joint model for both SMILES-based generation and prediction, while enjoying synergistic benefits,~remains~open.

\section{Background}\label{section:joint-modeling}

\paragraph{Problem Formulation}\label{section:problem-statement}

\label{paragraph:joint-modeling-problem-statement} 
The aim of \emph{joint modeling} is to learn the joint distribution of the data and its properties $p(\x, y)$, i.e.,\ to identify a model that at the same time generates new data and predicts its properties. We assume access to a \emph{labeled dataset} ${\D} = \{(\x_n, y_n)\}_{n=1}^{N}$,  sampled from the joint data distribution $p(\x, y)$, often accompanied with an \emph{unlabeled dataset} ${\D}_U = \{\x_n\}_{n=1}^{N_U}$, sampled from $p(\x)$. Here, examples $\x$ can be thought of as molecules and labels $y$ as molecular properties. 

In the general formulation of \citet{lasserre2006principled}, joint modeling aims to learn the joint distribution $p(\x, y)$ by defining a \emph{joint model} $p_{\theta, \phi}(\x, y)$ that factorizes into a \emph{generative model} $p_\theta(\x)$ and a \emph{predictive model} $p_{\phi}(y \mid \x)$ such that
\begin{equation}\label{equation:joint-model}
    p_{\theta, \phi}(\x, y) = p_{\phi}(y \mid \x)p_{\theta}(\x) ,
\end{equation} 
where $\theta$ denotes the parameters of the generative model, and $\phi$ %
the parameters of the predictive model. Training of the joint model is equivalent to minimizing the negative log-likelihood,
i.e., the \emph{joint loss} 
\begin{equation}\label{equation:joint-loss}
    \ell_{\lambda}(\theta, \phi) = -\E_{(\x, y) \sim p(\x, y)} [\ln p_\theta(\x) + \lambda \ln p_{\phi}(y \mid {\x})],
\end{equation}
where $\lambda \in \mathbb{R}$ weights the predictive and the generative parts.

Choosing the extent to which parameters $\theta$ and $\phi$ are shared and the way the joint loss is optimized,~is crucial for obtaining a model with both a high generative and predictive performance, at the same time maintaining the synergistic benefits of joint learning~\citep{lasserre2006principled}. 

\subsection{Transformer-based Models}\label{sect:transformers_joint_modeling}
Transformers \citep{vaswaniAttentionAllYou2017} achieve state-of-the-art performance in both molecule generation \citep{bagal2022molgpt} and property prediction \citep{zhou2023unimol} tasks.  

\paragraph{Transformer Encoders and Decoders}
Transformers used for generation and for property prediction differ in the use of the \emph{self-attention} mechanism. Transformer decoders, used for generative tasks, employ a \emph{causal self-attention}
\begin{equation}\label{equation:self-attention-causal}
    Att_\rightarrow(\Queries, \Keys, \Values) = \softmax\left(\frac{\Queries\Keys^T}{\sqrt{d}} + {\Mask}_\rightarrow \right)\Values,
\end{equation}
where $\Queries, \Keys, \Values \in \R^{T \times d}$ are \emph{query}, \emph{key} and \emph{value} matrices, respectively, ${\Mask}_{\rightarrow} \in \R^{T \times T}$ is a \emph{causal mask}, i.e.,\ a matrix such that $(\Mask_{\rightarrow})_{ij} = 0$, if $i \geq j$, and $(\Mask_{\rightarrow})_{ij} = -\infty$, otherwise, $T$ is the sequence length and $d$ is the head dimension.\footnote{We assume that the dimensions of the query, key, and value matrices are equal.} On the other hand, transformer encoders, used for predictive tasks, employ a \emph{bidirectional self-attention}%
\begin{equation}\label{equation:self-attention-bidirectional}
    Att_\leftrightarrow(\Queries, \Keys, \Values) = \softmax\left(\frac{\Queries\Keys^T}{\sqrt{d}} + {\Mask}_\leftrightarrow \right)\Values,
\end{equation}
where ${\Mask}_\leftrightarrow \in \R^{T \times T}$ is a \emph{bidirectional mask}, i.e., $({\Mask}_\leftrightarrow)_{ij} = 0$ for all $i,j \in [T]$.

\paragraph{Alternating attention} The definition of the transformer decoder and encoder can be generalized by using an alternating attention scheme~\citep{dong2019unified}:
\begin{equation}
        Att_{\MaskToken}(\Queries, \Keys, \Values) = \softmax\left(\frac{\Queries \Keys^T}{\sqrt{d}} + \mathbf{M}_{\MaskToken} \right) \Values,    
\end{equation}
where $\MaskToken \in\{\rightarrow, \leftrightarrow\}$ and  $\mathbf{M}_{\MaskToken} = \mathbf{M}_{\rightarrow}$ is a causal mask upon using the model as a transformer decoder and $\mathbf{M}_{\MaskToken} = \mathbf{M}_{\leftrightarrow}$, otherwise.

\paragraph{Training transformers}
Training transformers proceeds in a two-step manner, by first \emph{pre-training} the model on an unlabeled dataset and then \emph{fine-tuning} the pre-trained model on a downstream task. Transformer decoders and encoders are pre-trained using different losses. 

\paragraph{Pre-training} Transformer decoders, optimized for generative performance, are predominantly pre-trained using the negative log-likelihood loss $- \E_{\x \sim p(\x)} [\ln p_\theta(\x)]$. 
As the causal mask induces a factorization of the transformer decoder into an autoregressive model $p_\theta(\x) = \prod_{t=1}^T p_\theta(x_t \mid {\x}_{<t})$, where~$\x~=~(x_1, \ldots, x_T)$, the generative loss reduces to the \emph{language modeling} (LM) loss 
\begin{equation}\label{eq:language-modeling-loss}
\ell_{\textsc{LM}}(\theta) = - \E_{\x \sim p(\x)} \left[\sum_{t=1}^T \ln p_\theta(x_t \mid {\x}_{<t})\right]. 
\end{equation}
On the other hand, transformer encoders are usually pre-trained using \emph{masked language modeling} (MLM) loss
\begin{equation}\label{eq:masked-reconstruction-loss}
    \ell_{\textsc{MLM}}(\theta) =
    - \mathbb{E}_{\x \sim p(\x)} \mathbb{E}_{\mathcal{M}}\Big[ \ln p_{\theta}({\x}_{\mathcal{M}} \mid {\x}_{\mathcal{R}}) \Big], 
\end{equation}
where $\x = (x_1, \ldots, x_T)$, $\mathcal{M}$ is a set of indices drawn uniformly at random from the set of token indices $\{1, \ldots, T \}$ and the set of all tokens whose indices belongs to $\mathcal{M}$ are \emph{masked tokens} ${\x}_\mathcal{M}$. The rest of the tokens ${\x}_{\mathcal{R}}$ are defined such that $\x = {\x}_\mathcal{M} \cup {\x}_{\mathcal{R}}$.

\paragraph{Fine-tuning} Next, the pretrained model is fine-tuned by defining a predictive head on top of the pretrained model and training it as a predictor on a labeled dataset using the \emph{prediction loss}
\begin{equation}\label{equation:predictive-loss}
    \ell_{\textsc{pred}}(\phi) = - \E_{(\x, y) \sim p(\x, y)} [\ln p_\phi(y \mid \x)].
\end{equation} 

\begin{figure}
    \centering
    \includegraphics[width=0.5\linewidth]{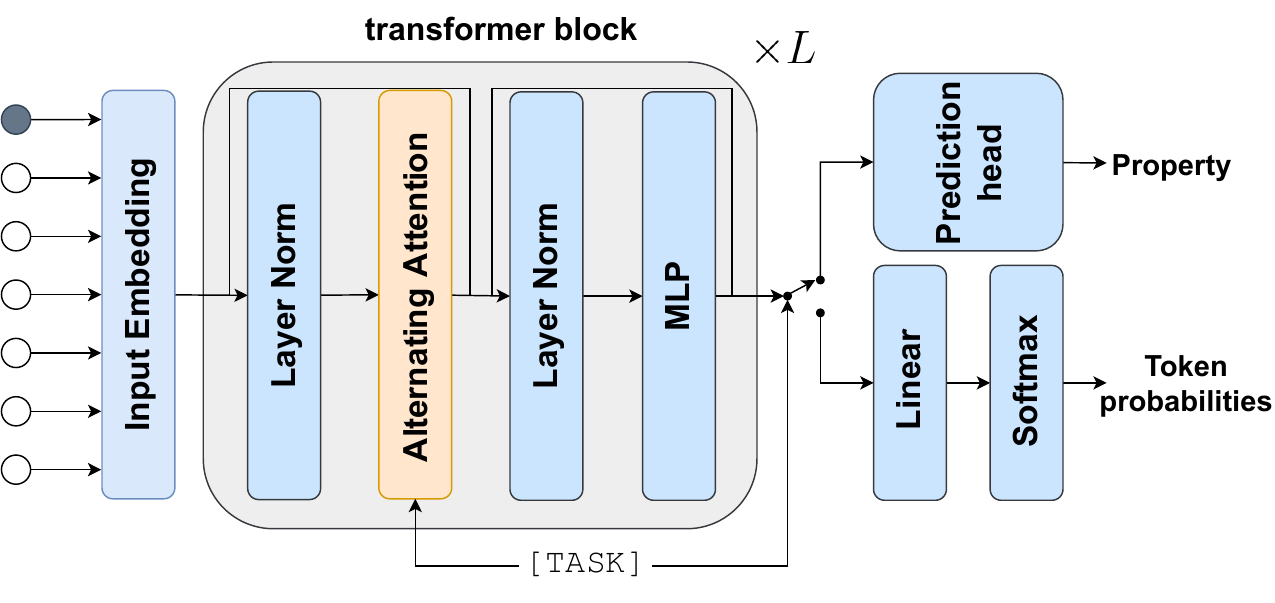}
    \caption{A schematic representation of $\HyF$. Depending on the task token $\TaskToken$, $\HyF$ uses either a causal or a bidirectional mask, outputting token probabilities or predicted property values.}
   \label{fig:jointformer-schematic}
\end{figure}

\section{Hyformer}\label{section:Hyformer}

We propose \textsc{Hyformer}, a joint transformer-based model that unifies a generative decoder with a predictive encoder in a single set of shared parameters, using an alternating training scheme.

\subsection{Model Formulation}

$\HyF$ unifies a decoder with an encoder using a transformer backbone $f_{\theta}(\x; \TaskToken)$ conditioned on a \emph{task token} $\TaskToken \in \{ \LMToken,\PredToken,\MLMToken \}$. The task token facilitates switching between respective losses during training (see Section~\ref{section:hybrid-transformer-training}) and determines whether the backbone $f_\theta$ processes input $\x$ in an autoregressive manner using a causal, or a bidirectional mask
\[
\MaskToken = 
\begin{cases}
  \rightarrow \quad & \text{if } \TaskToken = \LMToken, \\
  \leftrightarrow \quad & \text{if } \TaskToken \in \{ \PredToken, \MLMToken \}.
\end{cases}
\]

Finally, the generative $p_\theta(\x)$ and predictive  $p_\theta(y \mid \x)$ parts of the joint model, factorized as
\begin{align}
    &p_\theta(\x, y) \defeq p_\theta(\x) p_\theta(y \mid \x),
\end{align}
are implemented by adding a generative and a predictive head on the top of the shared backbone $f_\theta$.

\begin{algorithm}[h]
	\caption{Training of $\HyF$}
    \label{algorithm:jointformer-pretraining}
    \let\AND\relax
	\begin{algorithmic}[1]
		\REQUIRE Dataset $\mathcal{D}$ (labeled or unlabeled); model parameters $\theta$; task probabilities $\mathbf{p}_{\TaskToken}$.\\
        For pre-training: $\TaskToken \in \{\LMToken, \PredToken, \MLMToken\}$, for fine-tuning: $\TaskToken \in \{\LMToken, \PredToken\}$.
        \WHILE{stopping criterion not met}
        \STATE Sample task $\TaskToken \sim \textsc{Cat}(\mathbf{p}_{\TaskToken})$
        \STATE Select loss $\ell_{\TaskToken}$ and the corresponding attention mask
        \STATE Update model parameters $\theta$ using the gradient of $\ell_{\TaskToken}$
        \ENDWHILE
	\end{algorithmic} 
\end{algorithm}

\subsection{$\HyF$ Training}\label{section:hybrid-transformer-training}

As with standard transformer-based models, the training of $\HyF$ is divided into a pre-training and a fine-tuning stage.

\paragraph{Joint Pre-training}
To unify the generative and the predictive functionalities in a single model, we pre-train $\HyF$ using a variant of the joint loss (Eq.~\ref{equation:joint-loss}). For the generative part, we use the language modeling loss $\ell_{\textsc{LM}}$,
while for the predictive part, we use the masked language modeling loss $\ell_{\textsc{MLM}}$ and the predictive loss $\ell_{\textsc{pred}}$, with the combined loss being defined as:
\begin{equation}\label{eq:reconstruction-loss}
    \ell_{\textsc{Hyformer}} = \ell_{\textsc{LM}} + \mu \ell_{\textsc{MLM}} + \eta \ell_{\textsc{PRED}}.
\end{equation}
As \emph{pre-training labels}, we use values analytically computable from the input sequences, e.g., molecular descriptors, such as molecular weight for small molecules, or hydrophobicity for peptides. When the pre-training labels are not available, $\HyF$ is pre-trained without the predictive loss $\ell_{\textsc{PRED}}$. Analogously to multitask learning~\citep{raffel2023exploring}, the weighted loss $ \ell_{\textsc{Hyformer}}$ (Eq.~\ref{eq:reconstruction-loss}) is effectively implemented  using a vector of task probabilities $\mathbf{p}_{\TaskToken} = (p_{\LMToken}, p_{\MLMToken}, p_{\PredToken})$, which defines how the generative and predictive capabilities of the joint model are balanced.

During training, the shared parameters $\theta$ are updated differently depending on the task token. If $\TaskToken \in \{\PredToken, \MLMToken\}$, a bidirectional attention mask $\mathbf{M}_{\leftrightarrow}$ is applied and all attention module weights are updated, since the bidirectional mask does not restrict information flow. Conversely, if $\TaskToken = \LMToken$, a causal mask $\mathbf{M}_{\rightarrow}$ is applied, restricting each token to attend only to its left context, altering the gradients of the attention module, due to the functional form of the Jacobian of the softmax function, alleviating gradient interference typical for joint modeling (Appendix~\ref{appendix:proofs-gradient-interference}).

\paragraph{Fine-tuning} We fine-tune $\HyF$ using the joint loss (Eq.~\ref{equation:joint-loss}), defined as 
\begin{equation}
    \ell_{\textsc{Hyformer}} = \ell_{\textsc{LM}}+ \lambda \ell_{\textsc{pred}}.
\end{equation}
Analogously to pre-training, $\HyF$ alternates between the generative and predictive task, to balance their objectives, based on a pre-defined vector of task probabilities $\mathbf{p}_{\TaskToken}~=~(p_{\LMToken}, p_{\textsc{[pred]}})$. We assume that {\textit{fine-tuning labels}} used in loss $\ell_{\textsc{pred}}$ are different than in the pre-training phase and are defined by the downstream prediction task. Specifically, we omit the masked language modeling loss, to focus on the downstream task while retaining the generative capabilities of the model.

\subsection{Sampling}\label{section:sampling} 
Sampling from $\HyF$ exploits the generative $p_{\theta}(\x)$ and predictive part $p_{\theta}(y \mid \x)$ depending on the sampling mode: unconditional or conditional.

\paragraph{Unconditional Generation} In unconditional generation, we sample $\x \sim p_{\theta}(\x)$ using the autoregressive part of the model. This addresses a limitation of conditionally trained generative models \citep{bagal2022molgpt} and joint models trained without a pure unsupervised objective \citep{born2023regression}, where generating a single example requires conditioning on a fixed property value inferred from a dataset.  

\paragraph{Conditional Generation} To generate $(\x, y) \sim p_\theta(\x, y)$ that satisfies a condition $Y \subseteq \mathcal{Y}$, $\HyF$ samples $K$-many examples ${\x}_1, \ldots, {\x}_K \sim \generator$ and, for every $k=1, \ldots, K$, accepts sample ${\x}_k$, if the predictor $p_\theta(y \mid \x)$ classifies ${\x}_k$ as having property $Y$. As a simple consequence of the Bayes rule, the above procedure yields a correct conditional sampling procedure, as
\begin{equation}
    p(\x \mid y \in Y) \propto p(y \in Y \mid \x)p(\x),
\end{equation}
for $y \in Y \subseteq \mathcal{Y}$ such that $p(y \in Y) >0$. Note that the conditional sampling procedure of $\HyF$ is a variant of best-of-$K$ sampling, a provably near-optimal solution to the KL-regularized RL problem \citep{yang2019analyzing} that has been shown to outperform other conditional sampling methods for LLMs, including state-of-the-art reinforcement learning methods like PPO and DPO \citep{touvron2023llama, mudgal2023controlled, gao2023scaling, rafailov2023direct}. Crucially, $\HyF$ leverages a jointly trained predictor $p_\theta(y \mid x)$ over a unified representation space, resulting in tighter alignment between generation and control. This coherence is particularly valuable in drug discovery, where the primary objective is not throughput, but \emph{precision and sample efficiency}, that is, generating a small number of high-quality candidates with minimal false positives.

\section{Experiments}\label{section:experiments}

We evaluate $\HyF$ across a broad range of molecular modeling tasks. First, we demonstrate the synergistic benefits of joint modeling in three settings: (i) conditional generation on GuacaMol dataset \citep{brown2019guacamol}, (ii) out-of-distribution (OOD) property prediction on Hit Identification task from the Lo-Hi benchmark \citep{steshin2023lohipracticalmldrug} and (iii) representation learning via probing on MoleculeNet benchmark \citep{wu2018moleculenet}. Subsequently, we show that $\HyF$ rivals state-of-the-art generative and predictive models in both unconditional generation on GuacaMol and property prediction on MoleculeNet. Finally, we apply $\HyF$ to antimicrobial peptide (AMP) design, showcasing the benefits of our joint modeling approach. Experimental details and additional results are provided in Appendix~\ref{appendix:pre-training}, \ref{appendix:experimental-details}~and~\ref{appendix:additional-experiments}.

\subsection{Synergistic Benefits of $\HyF$}\label{subsection:experiment-set-1}

\begin{table}[!h]
\vspace{-4ex}
\caption{Conditional generative performance on GuacaMol dataset. Best model is marked 
\textbf{bold}.}
\vspace{-0.55ex}
\begin{center}
\begin{scriptsize}
\begin{sc}
\begin{tabular}{l@{}clcccc}
\toprule
Model & Joint & Metric & QED & SA & logP & Avg. \\
\midrule
\multirow{3}{*}{MolGPT} & \multirow{3}{*}{\ding{55}} 
& MAD $\downarrow$ & 0.087 & 0.019 & 0.276 & 0.127 \\
& & SD $\downarrow$ & 0.074 & 0.017 & 0.262 & 0.118 \\
& & Validity $\uparrow$ & 0.985 & 0.986 & 0.982 & 0.984 \\
\midrule
\multirow{3}{*}{GraphGPT} & \multirow{3}{*}{\ding{55}} 
& MAD $\downarrow$ & 0.039 & 0.011 & 0.158 & 0.069 \\
& & SD $\downarrow$ & 0.082 & 0.047 & 0.653 & 0.261 \\
& & Validity $\uparrow$ & \textbf{0.998} & \textbf{0.997} & \textbf{0.992} & \textbf{0.995} \\
\midrule
\multirow{6}{*}{Hyformer} & \multirow{3}{*}{\ding{55}} 
& MAD $\downarrow$ & \hha 0.031 (0.003) & \hha 0.015 (0.001) & \hha 0.131 (0.010) & \hha 0.059 (0.004) \\
& & SD $\downarrow$ &\hha 0.045 (0.004) & \hha 0.020 (0.001) & \hha 0.170 (0.014) & \hha 0.078 (0.006) \\
& & Validity $\uparrow$ &\hha 0.993 (0.003) & \hha 0.990 (0.004) & \hha 0.985 (0.014) & \hha 0.989 (0.007) \\
\cmidrule(lr){2-7}
& \multirow{3}{*}{\ding{51}} 
& MAD $\downarrow$ &\hh \textbf{0.008 (0.001)} & \hh \textbf{0.005 (0.000)} & \hh \textbf{0.041 (0.002)} & \hh \textbf{0.018 (0.001)} \\
& & SD $\downarrow$ &\hh \textbf{0.015 (0.002)} & \hh \textbf{0.009 (0.002)} & \hh \textbf{0.051 (0.004)} & \hh \textbf{0.025 (0.003)} \\
& & Validity $\uparrow$ &\hh 0.990 (0.007) & \hh 0.985 (0.003) & \hh 0.987 (0.006) & \hh 0.987 (0.005) \\
\bottomrule
\end{tabular}
\end{sc}
\end{scriptsize}
\end{center}
\vspace{-5mm}
\end{table}

\subsubsection{Conditional Molecule Generation}\label{subsubsection:experiment-conditional-molecule-generation}

To demonstrate the synergistic benefits of $\HyF$ in generating molecules with specific molecular properties, we follow the setup of~\citet{bagal2022molgpt} and jointly pre-train $\HyF$ scaled to 8.5M parameters on GuacaMol dataset with 1.3M molecules, using pre-computed molecular descriptors~\citep{yang2019analyzing}. We subsequently jointly fine-tune $\HyF$ on GuacaMol dataset with QED, SA, and LogP molecular properties, as fine-tuning labels, and generate molecules with specific properties using $\HyF$’s conditional sampling procedure. Pre-training and experimental details alongside results for all property settings can be found in Appendix~\ref{appendix:pre-training} and \ref{appendix:conditional-molecule-generation-task}.

Following~\citep{gao2024graph}, we compare $\HyF$ to MolGPT~\citep{bagal2022molgpt} and GraphGPT~\citep{gao2024graph} using: mean absolute deviation (MAD) from the target property value, standard deviation (SD) of the generated property values and validity of the generated molecules. Evaluation is averaged across three target values per each property: QED:\{0.5, 0.7, 0.9\}, SA:\{0.7, 0.8, 0.9\}, and logP:\{0.0, 2.0, 4.0\}. Additionally, we compare to a non-joint variant of $\HyF$, in which the predictive head is fine-tuned with prediction loss, on top of a frozen, pre-trained generative part, i.e., without joint fine-tuning.

The jointly fine-tuned $\HyF$ achieves the lowest MAD and SD across all properties, while maintaining high validity, outperforming all baselines. Notably, $\HyF$ improves controllability over it's non-joint counterpart, confirming that joint fine-tuning enhances conditional generation. Although GraphGPT attains slightly higher validity, it does so at the cost of reduced controllability. These results demonstrate that joint modeling enables robust property-conditioned molecular generation across a range of chemically relevant~targets.

\subsubsection{Out-of-Distribution Molecular Property Prediction%
}\label{subsubsection:experiments-hit-identification}

To evaluate the ability of $\HyF$ to predict molecular properties in an out-of-distribution (OOD) setting,~we jointly pre-train $\HyF$ scaled to 50M parameters on 19M molecules from \citep{zhou2023unimol}, together with pre-computed molecular descriptors~\citep{yang2019analyzing}, and benchmark on the Hit Identification (Hi) task from the Lo-Hi benchmark \citep{steshin2023lohipracticalmldrug}. The Hi task requires generalization to molecular scaffolds not seen during training, with the test set constructed such that no molecule has a Tanimoto similarity greater than 0.4 (based on ECFP4 fingerprints) to any molecule in the training set. This setup mimics realistic drug discovery scenarios, where generalization beyond known chemical space is essential. For experimental details,~see Appendix~\ref{appendix:pre-training}~and~\ref{appendix:ood-property-prediction-task}. 

We follow the setup of \citep{steshin2023lohipracticalmldrug} and compare jointly fine-tuned $\HyF$ to all models reported in \citep{steshin2023lohipracticalmldrug}; machine learning models: k-NN, gradient boosting (GB), SVM and MLP, trained on molecular fingerprints (ECFP4, MACCS) and deep learning models: Chemformer \citep{yang2019analyzing} and Graphformer \citep{ying2021transformers, shi2022benchmarking}. Moreover, we compare to $\HyF$ (no-joint), which is a version of our model pre-trained using MLM loss, hence without alternating attention, and fine-tuned using the prediction loss only.

$\HyF$ achieves the highest mean AUPRC across all datasets (Table~\ref{table:hi-task}), outperforming fingerprint-based baselines and demonstrating the potential of deep learning methods in real-world drug discovery. The consistent ranking in favor of $\HyF$ shows the benefits of joint modeling in out-of-distribution molecular property prediction, although the differences are not statistically significant at the 95\% confidence level.

\begin{table}[t]
\vspace{-3ex}
\caption{Predictive performance (AUPRC) on Hit Identification (Hi) task from Lo-Hi benchmark. 
Mean and standard deviation across 3 random seeds.}
\label{table:hi-task}
\begin{center}
\begin{scriptsize}    
\begin{sc}
\begin{tabular}{lcccc}
\toprule
& \multicolumn{4}{l}{Dataset, AUPRC ($\uparrow$)}\\
\cmidrule(lr){2-5}
Model & DRD2-Hi & HIV-Hi & KDR-Hi & Sol-Hi \\
\midrule
Dummy baseline & 0.677 (0.061) & 0.040 (0.014) & 0.609 (0.081) & 0.215 (0.008) \\
KNN (ECFP4) & 0.706 (0.047) & 0.067 (0.029) & 0.646 (0.048) & 0.426 (0.022) \\
KNN (MACCS) & 0.702 (0.042) & 0.072 (0.036) & 0.610 (0.072) & 0.422 (0.009) \\
GB (ECFP4) & 0.736 (0.050) & 0.080 (0.038) & 0.607 (0.067) & 0.429 (0.006) \\
GB (MACCS) & 0.751 (0.063) & 0.058 (0.030) & 0.603 (0.074) & 0.502 (0.045) \\
SVM (ECFP4) & 0.677 (0.061) & 0.040 (0.014) & 0.611 (0.081) & 0.298 (0.047) \\
SVM (MACCS) & 0.713 (0.050) & 0.042 (0.015) & 0.605 (0.082) & 0.308 (0.021) \\
MLP (ECFP4) & 0.717 (0.063) & 0.049 (0.019) & 0.626 (0.047) & 0.403 (0.017) \\
MLP (MACCS) & 0.696 (0.048) & 0.052 (0.018) & 0.613 (0.077) & 0.462 (0.048) \\
\midrule
Chemprop & 0.782 (0.062) & 0.148 (0.114) & 0.676 (0.026) & 0.618 (0.030) \\
Graphormer & 0.729 (0.039) & 0.096 (0.070) & - & - \\
\hha \textsc{Hyformer} (no-joint) & \hha 0.778 (0.070) & \hha 0.154 (0.108) & \hha 0.675 (0.046) & \hha 0.601 (0.040) \\
\hh \textsc{Hyformer} & \hh 0.784 (0.082) & \hh 0.158 (0.128) & \hh 0.701 (0.022) & \hh 0.640 (0.036) \\
\bottomrule
\end{tabular}
\end{sc}
\end{scriptsize}
\end{center}

        \vspace{-3ex}
\end{table}

\subsubsection{Molecular Representation Learning}\label{subsubsection:experiment-probing}
\begin{table}[h]
\caption{Molecular representation learning performance of predictive, generative and joint models on MoleculeNet benchmark, evaluated using linear and KNN probing. Best model within each probing method is marked \textbf{bold}.}
\label{table:molecular-representation-learning-task}
\begin{center}
\begin{scriptsize}
\begin{sc}
\resizebox{\linewidth}{!}{%
\begin{tabular}{cl@{}lcccccccccc}
\toprule
& & & \multicolumn{3}{l}{Dataset, RMSE $\downarrow$} & \multicolumn{7}{l}{Dataset, AUCROC $\uparrow$}  \\
\cmidrule(lr){4-6} \cmidrule(lr){7-13}
  & Type \, &        Model &                       Esol &                   Freesolv &                       Lipo &                      BBBP &                      BACE &                   ClinTox &                     Tox21 &                   ToxCast &                     SIDER &                       HIV \\
\midrule
\parbox[t]{2mm}{\multirow{7}{*}{\rotatebox[origin=c]{90}{Linear}}}
& P. & Uni-Mol & 1.350 & \textbf{2.503} & 1.002 & 65.5 & 66.3 & 74.3 & 70.1 & 59.9 & 58.1 & 73.6 \\
& P. & \hha $\HyF$ (no-joint) & \hha 1.256 & \hha 2.640 & \hha 0.894 & \hha 68.4 & \hha 73.6 & \hha 98.8 & \hha \textbf{73.4} & \hha \textbf{61.2} & \hha 58.8 & \hha 75.9 \\
& G. & MolGPT & 1.299 & 4.110 & 1.033 & 66.8 & 79.1 & 97.8 & 71.9 & 60.5 & 59.2 & \textbf{77.5} \\
& J. & MoLeR & \textbf{1.223} & 4.935 & 0.938 & 67.8 & \textbf{79.5} & 84.6 & 71.1 & 59.3 & 58.3 & 74.6 \\
& J. & RT & 2.510 & 4.515 & 1.158 & 54.7 & 63.1 & 57.3 & 50.5 & 52.8 & 54.5 & 65.6 \\
& J. & Graph2Seq & 1.498 & 3.486 & 0.890 & 66.0 & 76.7 & 72.0 & 71.2 & 60.4 & 50.5 & 57.1 \\
& J. & \hh $\HyF$ & \hh 1.527 & \hh 4.294 & \hh \textbf{0.887} & \hh \textbf{68.5} & \hh 77.2 & \hh \textbf{99.5} & \hh 72.4 & \hh 60.7 & \hh \textbf{60.8} & \hh 74.7 \\
\midrule
\midrule
\parbox[t]{2mm}{\multirow{7}{*}{\rotatebox[origin=c]{90}{KNN}}}
& P. & Uni-Mol & 1.579 & 3.403 & 1.025 & 60.0 & 75.9 & 78.0 & 64.7 & 57.5 & 61.0 & 64.3 \\
& P. & \hha $\HyF$ (no-joint) & \hha 1.380 & \hha 3.254 & \hha 0.978 & \hha 67.8 & \hha 75.4 & \hha 89.0 & \hha 66.3 & \hha 57.6 & \hha 58.1 & \hha 71.4 \\
& G. & MolGPT & \textbf{1.232} & \textbf{3.075} & 0.987 & 68.4 & 71.9 & \textbf{94.2} & 66.0 & 56.9 & 61.0 & 70.5 \\
& J. & MoLeR & 1.802 & 4.061 & 1.096 & 59.4 & 72.0 & 71.2 & 64.9 & 53.3 & 57.3 & 67.3 \\
& J. & RT & 2.411 & 4.734 & 1.242 & 59.3 & 56.1 & 59.4 & 50.8 & 52.2 & 51.2 & 54.1 \\
& J. & Graph2Seq & 1.361 & 3.796 & 0.967 & \textbf{71.0} & \textbf{80.6} & 56.3 & 67.7 & 57.8 & 49.9 & 52.4 \\
& J. & \hh $\HyF$ & \hh 1.260 & \hh 3.999 & \hh \textbf{0.902} & \hh 69.5 & \hh 78.4 & \hh 93.8 & \hh \textbf{71.2} & \hh \textbf{59.3} & \hh \textbf{64.1} & \hh \textbf{71.8} \\

\bottomrule
\end{tabular}
}
\end{sc}
\end{scriptsize}
\end{center}
\end{table}

To assess the quality of molecular representations learned by $\HyF$, we introduce a novel probing protocol that emulates a typical drug discovery setting, where fixed molecular embeddings are used as inputs to downstream predictive models. In this setup, we train simple linear models with L2 regularization, and k-nearest neighbor (KNN) predictors on the top of frozen embeddings extracted from the respective pre-trained models. To ensure comparability with MoleculeNet benchmark (Section~\ref{section:molecular-property-prediction}), we reuse the same datasets, data splits, and model checkpoints. Implementation details are provided in Appendix~\ref{appendix:representation-learning-task}.

We compare representations extracted from jointly pre-trained $\HyF$ to those extracted from a range of baselines, including state-of-the-art generative (MolGPT~\citep{bagal2022molgpt}), predictive (Uni-Mol~\citep{zhou2023unimol}), and joint models: MoLeR~\citep{maziarz2022learning}, Regression Transformer (RT) \citep{born2023regression} and Graph2Seq \citep{gao2024graph}. Moreover, to quantify the effect of alternating attention and joint pre-training, we compare to $\HyF$ (no-joint), the version of our model trained solely with MLM loss.  

The jointly pre-trained representations from $\HyF$ are the most predictive across both KNN and linear probings, achieving the best performance on 4 out of 10 datasets for linear, and 5 out of 10 datasets for KNN, outperforming all other baselines (Table~\ref{table:molecular-representation-learning-task}). An additional analysis of linear probing on ClinTox shows high per-target F1 scores of 0.98 and 0.90, indicating robust performance of \textsc{Hyformer} across targets. The next best models, Hyformer (no-joint) and MoLeR for linear and MolGPT for KNN probing, rank first on 2 and 3 out of 10 datasets, respectively. Notably, joint models outperform UniMol, the state-of-the-art property predictor, on all datasets, except for Freesolv with linear probing, highlighting the effectiveness of joint modeling for transferable molecular representation learning.

\subsection{Generative and predictive performance of $\HyF$}\label{subsection:experiment-set-2}

We next confirm that $\HyF$ effectively addresses the challenges of joint training, while it enjoys the synergistic benefits described above, it does not sacrifice generative or predictive performance compared to state-of-the-art models trained separately for these tasks.

\subsubsection{Unconditional Molecule Generation}\label{section:unconditional-molecule-generation}

\begin{table}[!t]
\vspace{-3.5ex}
\caption{Unconditional generative performance on GuacaMol distribution learning benchmarks.
The best model in each category is marked \textbf{bold}.}
\label{table:guacamol-distribution-learning-benchmark}
\begin{center}
\begin{scriptsize}    
\begin{sc}
\begin{tabular}{@{}lccccc@{}}
    \toprule
    Model & FCD Score $\uparrow$ & KL div.\ Score $\uparrow$ & Val.\ $\uparrow$ & Uniq.\ $\uparrow$ & Nov.\ $\uparrow$ \\
    \midrule
    \textit{Graph-based} \vspace{0.3ex} \\ 
    JT-VAE & 0.750 & 0.940 & \textbf{1.000} & - & - \\
    MoLeR & 0.625 & 0.964 & \textbf{1.000} & \textbf{1.000} & \textbf{0.991} \\
    MAGNet & \textbf{0.760} & 0.950 & \textbf{1.000} & - & - \\
    MiCaM & 0.731 & \textbf{0.989} & \textbf{1.000} & 0.994 & 0.986 \\
    \midrule
    \textit{SMILES-based} \vspace{0.3ex} \\
    VAE & 0.863 & 0.982 & 0.870 & 0.999 & 0.974 \\
    LSTM & 0.913 & 0.991 & 0.959 & \textbf{1.000} & 0.912 \\
    MolGPT & 0.907 & 0.992 & 0.981 & 0.998 & \textbf{1.000} \\
    \hh $\textsc{Hyformer}_{\tau=0.9}$  & \hh 0.897 (0.002) & \hh \textbf{0.995 (0.000)} & \hh \textbf{0.986 (0.001)} & \hh 0.999 (0.000) & \hh 0.879 (0.006) \\
    \hh $\textsc{Hyformer}_{\tau=1.0}$ & \hh \textbf{0.918 (0.002)} & \hh 0.989 (0.001) & \hh 0.978 (0.000) & \hh 0.999 (0.000) & \hh 0.908 (0.002) \\
    \hh $\textsc{Hyformer}_{\tau=1.1}$ & \hh 0.894 (0.002) & \hh 0.977 (0.001) & \hh 0.965 (0.001) & \hh \textbf{1.000 (0.000)} & \hh 0.931 (0.001) \\
    \bottomrule
\end{tabular}
\end{sc}
\end{scriptsize}
\end{center}
\vspace{-2ex}
\end{table}

To evaluate the unconditional generative performance of $\HyF$, we perform an evaluation on the Guacamol distribution learning benchmark~\citep{brown2019guacamol}. We use $\HyF$ scaled to 8.5M parameters and trained on GuacaMol dataset with $1.3$M molecules, together with pre-computed molecular descriptors~\citep{yang2019analyzing}, and investigate the impact of sampling temperature $\tau$. For experimental details, see Appendix~\ref{appendix:molecule-generation-task}. 

We compare to state-of-the-art unconditional generative models; SMILES-based: VAE \citep{kingma2013auto}, LSTM \citep{gers2001lstm}, MolGPT \citep{bagal2022molgpt} and graph-based: JT-VAE \citep{jin2019junction}, MoLeR \citep{maziarz2022learning}, MAGNet \citep{hetzel2023magnet}, MiCaM \citep{geng2023novo}. We omit RT \citep{born2023regression} and GraphGPT \citep{gao2024graph} as they do not generate molecules unconditionally or provide results on the GuacaMol benchmark. 

$\HyF$, with top FCD and KL div.\ score values, outperforms graph-based models, while achieving the highest validity among SMILES-based models. Across various sampling temperatures $\tau$, $\HyF$ consistently lies on the Pareto front, balancing distributional fidelity (FCD Score, KL div.\ Score), validity and uniqueness. Overall, SMILES-based models outperform those based on theoretically more informative graph representations in terms of FCD Score, at the expense of not always sampling~valid~molecules.

\subsubsection{Molecular Property Prediction}\label{section:molecular-property-prediction}

\begin{table}[!t]
\caption{Predictive performance of predictive and joint models on the MoleculeNet benchmark. Mean and standard\ deviation\ across 3 random seeds. The best model in each category, statistically significant at the 95\% confidence level, is marked~\textbf{bold}.}
\label{table:molecular-property-prediction-task}
\vspace{-2ex}
\begin{center}
\begin{scriptsize}
\begin{sc}
\resizebox{\linewidth}{!}{%
\begin{tabular}{clcccccccccc}
\toprule
& & \multicolumn{3}{l}{Dataset, RMSE $\downarrow$} & \multicolumn{7}{l}{Dataset, AUCROC $\uparrow$} \\
\cmidrule(lr){3-5} \cmidrule(lr){6-12}
& Model & Esol & Freesolv & Lipo & BBBP & BACE & ClinTox & Tox21 & ToxCast & SIDER & HIV \\ 
\midrule
\parbox[t]{2mm}{\multirow{11}{*}{\tiny\rotatebox[origin=c]{90}{Predictive}}} & D-MPNN & 1.050(0.008) & 2.082(0.082) & 0.683(0.016) & 71.0(0.3) &80.9(0.6)& 90.6(0.6)& 75.9(0.7)& 65.5(0.3)& 57.0(0.7)& 77.1(0.5) \\
& Attentive FP & 0.877(0.029) & 2.073(0.183) & 0.721(0.001) & 64.3(1.8)& 78.4(0.02)& 84.7(0.3)& 76.1(0.5)& 63.7(0.2)& 60.6(3.2)& 75.7(1.4) \\
& N-GramRF & 1.074(0.107) & 2.688(0.085) & 0.812(0.028) &69.7(0.6) &77.9(1.5)& 77.5(4.0)& 74.3(0.4)& - &66.8(0.7)& 77.2(0.1) \\
& N-GramXGB & 1.083(0.082) & 5.061(0.744) & 2.072(0.030) &69.1(0.8) &79.1(1.3) &87.5(2.7) &75.8(0.9)& - &65.5(0.7) &78.7(0.4) \\
& PretrainGNN & 1.100(0.006) & 2.764(0.002) & 0.739(0.003) &68.7(1.3)& 84.5(0.7)& 72.6(1.5) &78.1(0.6)& 65.7(0.6)& 62.7(0.8)& 79.9(0.7) \\
& GROVERbase & 0.983(0.090) & 2.176(0.052) & 0.817(0.008) &70.0(0.1)& 82.6(0.7)& 81.2(3.0)& 74.3(0.1)& 65.4(0.4)& 64.8(0.6)& 62.5(0.9) \\
& GROVERlarge & 0.895(0.017) & 2.272(0.051) & 0.823(0.010) &69.5(0.1)& 81.0(1.4) &76.2(3.7) &73.5(0.1) &65.3(0.5) &65.4(0.1) &68.2(1.1) \\
& GraphMVP & 1.029(0.033) & - & 0.681(0.010) & 72.4(1.6) &81.2(0.9) &79.1(2.8) &75.9(0.5) &63.1(0.4) &63.9(1.2)& 77.0(1.2) \\
& MolCLR & 1.271(0.040) & 2.594(0.249) & 0.691(0.004) &72.2(2.1) &82.4(0.9)& 91.2(3.5)& 75.0(0.2)& -& 58.9(1.4)& 78.1(0.5) \\
& Mole-BERT & 1.015 (0.030) & - & 0.676 (0.017)&  71.9 (1.6) & 80.8 (1.4) &  78.9 (3.0) & 76.8 (0.5) & 64.3 (0.2) & - & - \\
& GEM & 0.798(0.029) & 1.877(0.094) & 0.660(0.008) &72.4(0.4)& 85.6(1.1) &90.1(1.3) &78.1(0.1) &69.2(0.4) &67.2(0.4) &80.6(0.9) \\
& Uni-Mol & 0.788(0.029) & \textbf{1.480(0.048)} & \textbf{0.603(0.010)} & 72.9(0.6) & 85.7(0.2) & 91.9(1.8) & \textbf{79.6(0.5)} & 69.6(0.1) & 65.9(1.3) & 80.8(0.3) \\
\midrule
\parbox[t]{2mm}{\multirow{2}{*}{\tiny\rotatebox[origin=c]{90}{Joint}}} & Graph2Seq & 0.860(0.024) & 1.797(0.237) & 0.716(0.019) & 72.8(1.5) & 83.4(1.0) & - & 76.9(0.3) & 65.4(0.5) & 68.2(0.9) & 79.4(3.9) \\
& \hh \textsc{Hyformer} & \hh \textbf{0.774(0.026)} & \hh 2.047(0.076) & \hh \textbf{0.643(0.002)} & \hh 75.9(0.9) & \hh 83.8(1.1) & \hh \textbf{99.2(0.5)} & \hh \textbf{79.2(0.1)} & \hh 65.5(0.6) & \hh 65.7(1.6) & \hh 80.0(1.0) \\
\bottomrule
\end{tabular}
}
\end{sc}
\end{scriptsize}
\end{center}
\end{table}

To evaluate the predictive performance of $\HyF$, we use $\HyF$ scaled to 50M parameters on 19M molecules from \citep{zhou2023unimol}, together with pre-computed molecular descriptors~\citep{yang2019analyzing}, and fine-tune end-to-end on MoleculeNet benchmark \citep{wu2018moleculenet}. For experimental details, see Appendix~\ref{appendix:property-prediction-task}.

We follow the experimental protocol of \citep{zhou2023unimol}, use scaffold splitting and compare to predictive models: D-MPNN \citep{yang2019analyzing}, AttentiveFP \citep{xiong2019pushing}, N-gram \citep{liu2019n} with Random Forest and XGBoost \citep{chen2016xgboost},
PretrainGNN \citep{hu2019strategies}, GROVER \citep{rong2020self}, MolCLR \citep{wang2022molecular}, Mole-BERT \citep{xia2023molebert}, GraphMVP \citep{liu2021pre}, GEM \citep{fang2022geometry}, UniMol \citep{zhou2023unimol} and a joint model: Graph2Seq \citep{gao2024graph}. We omit RT \citep{born2023regression} and other models that use random splitting.

$\HyF$ obtains the lowest RMSE on Esol, and highest AUROC on BBBP and ClinTox, outperforming all models on 3 out of 10 datasets  (Table~\ref{table:molecular-property-prediction-task}). Moreover,~$\HyF$ performs better than Graph2Seq, the only other joint model capable of simultaneous molecule generation and property prediction, on 8 out of 10 datasets. Altogether, $\HyF$ outperforms the other joint learning model, Graph2Seq, and successfully rivals the performance of purely predictive models, demonstrating the efficiency of our joint learning strategy. 

\vspace{-2ex}
\begin{table}[h]
\caption{Conditional generative performance on antimicrobial peptide design. Mean and standard deviation computed over 100 bootstrap iterations. The best model is marked \textbf{bold}.}
\begin{center}
\begin{scriptsize}
\begin{sc}
\begin{tabular}{lccccccc}
    \toprule
    Model & Perplexity\footnotemark & Diversity $\uparrow$ & Fitness $\uparrow$ & HydrAMP$_{\textsc{MIC}}$ $\uparrow$ & AMPlify $\uparrow$ & amPEPpy $\uparrow$ \\
\midrule
PepCVAE
& 20.11 (0.14)
& \textbf{0.87} (0.0003)
& 0.07 (0.0004)
& 0.20 (0.0016)
& 0.49 (0.0016)
& 0.52 (0.0007) \\
AMPGAN
& 18.58 (0.10)
& 0.81 (0.0005)
& 0.12 (0.0005)
& 0.32 (0.0019)
& 0.64 (0.0018)
& 0.54 (0.0008) \\
HydrAMP
& 20.14 (0.12)
& \textbf{0.86} (0.0004)
& 0.09 (0.0003)
& 0.49 (0.0021)
& 0.59 (0.0016)
& 0.52 (0.0006) \\
AMP-Diffusion
& 16.93 (0.18)
& 0.82 (0.0004)
& 0.13 (0.0005)
& 0.26 (0.0018)
& 0.20 (0.0014)
& 0.38 (0.0006) \\
\hh \textsc{Hyformer}
& \hh 17.98 (0.06)
& \hh 0.80 (0.0005)
& \hh \textbf{0.19} (0.0006)
& \hh \textbf{0.80} (0.0019)
& \hh \textbf{0.94} (0.0027)
& \hh \textbf{0.72} (0.0018) \\

    \bottomrule
\end{tabular}
\end{sc} 
\end{scriptsize}
\end{center}
\vspace{-3ex}
\label{table:antimicrobial-peptide-design}
\end{table}
\footnotetext{We report perplexity, but do not seek to minimize it, as it inherently balances plausibility and novelty.}

\subsection{Antimicrobial Peptide Design}\label{section:experiments:antimicrobial-peptide-design}

To show the benefits of joint learning in a real-world use case related to drug discovery, we apply $\HyF$ to the task of antimicrobial peptide (AMP) design \citep{chen2023ampdiffusion}, i.e., generating AMPs with low minimal inhibitory concentration values (MIC) against \emph{E.~coli} bacteria. We jointly pre-train $\HyF$ on 3.5M general-purpose peptide sequences, and subsequently on 1M AMP sequences, together with 39 physicochemical descriptors from \emph{peptidy} package \citep{ozccelik2025peptidy}. Next, we jointly fine-tune $\HyF$ on 4,547 peptides with their MIC values~\citep{Szymczak2022.01.27.478054} and conditionally sample 50K peptides with an MIC regressor threshold set to $\leq 10^{0.3} \approx \SI{2}{\micro\molar}$. For experimental details, see Appendix~\ref{appendix:antimicrobial-peptide-design}.

We compare $\HyF$ AMP generation baselines: PepCVAE~\citep{das2018pepcvaesemisupervisedtargeteddesign}, AMPGAN~\citep{AMPGAN}, HydrAMP~\citep{Szymczak2022.01.27.478054}, and AMP-Diffusion~\citep{chen2023ampdiffusion}. Evaluation is based on four criteria: Perplexity~\citep{Torres2025.01.31.636003}, Diversity and Fitness~\citep{Li2024-pa}, and success rates in generating AMPs and low-MIC candidates. For the latter, we use HydrAMP$_{\textsc{MIC}}$, Amplify~\citep{Li2022-us}, and amPEPpy~\citep{10.1093/bioinformatics/btaa917} classifiers as state-of-the-art \emph{in-silico} oracles. 

$\HyF$ outperforms all baseline models by a large margin in terms of  generating peptides with a high fitness and AMP probability, as evaluated by all oracle classifiers (Table~\ref{table:antimicrobial-peptide-design}). Despite the stringent conditioning MIC threshold of $\SI{2}{\micro\molar}$, $\HyF$ maintains competitive perplexity and high diversity. These results suggest that even when constrained to explore less charted regions of sequence space, $\HyF$ is able to generate biologically plausible and novel peptide candidates.

To further validate the biological relevance of the generated peptides, we show that both unconditional sampling from jointly pre-trained $\HyF$, and conditional sampling from the fine-tuned model produces amino-acid distributions in close agreement with the training data (Figure~\ref{fig:hyformer_AMP}a).  Despite this very close agreement, the conditionally sampled peptides obtain a significant improvement of charge, aromaticity, and isoelectric point over the known non-AMPs, as compared to known AMPs (Fig.~\ref{fig:hyformer_AMP}b). Finally, to gain insight into which amino acids contribute most to antimicrobial activity, we analyze the attention weights of $\HyF$  (Fig.~\ref{fig:hyformer_AMP}c). The attention mechanism frequently prioritizes highly charged Arginine (R) and Lysine (K), which is expected as high AMP activity is associated with increased charge. The high attention frequency on Tryptophan (W) agrees with previous reports about this amino-acid's unique  ability to interact with the interface of the bacterial membrane~\citep{bi2014antimicrobial}. Finally, the high attention that $\HyF$ puts on Proline (P) agrees with the known high potency of Proline-rich AMPs, which kill bacteria  via a specific, non-lytic mechanism~\citep{lai2019identification}.

\section{Discussion}

In this paper, we introduced $\HyF$, a transformer-based joint model that combines an autoregressive decoder and a bidirectional encoder within a single set of shared parameters, using an alternating attention mechanism and joint pre-training. We showed that $\HyF$ provides synergistic benefits in conditional sampling, representation learning and out-of-distribution property prediction, with ablations highlighting the specific contributions of alternating attention and joint training. Furthermore, we validated the utility of joint modeling in a real-world antimicrobial peptide design task. Our results indicate that $\HyF$ successfully unifies molecular generation and property prediction for SMILES-based molecular representations, opening the avenue for the integration into real-world drug discovery pipelines, where informative molecular representations, robustness to OOD examples and robust conditional sampling are crucial. 

\paragraph{Limitations \& Future Work} 
However, joint modeling introduces an inherent trade-off. While shared parameters promote synergistic benefits and learning unified representations, they may limit task-specific specialization. Therefore, a promising direction for future work is designing dynamic or modular attention architectures that allocate capacity across tasks more flexibly, while preserving synergistic benefits. Moreover, to ensure fair comparison with prior work and isolate the effect of joint learning, we deliberately restricted model scale and relied on a fixed set of analytically computed descriptors. The extent to which the observed synergistic benefits carry over to other modalities, such as 3D structures, morphology or transcriptomics, remains an open question. 

\begin{figure*}[!t]
  \centering
  \includegraphics[width=\textwidth]{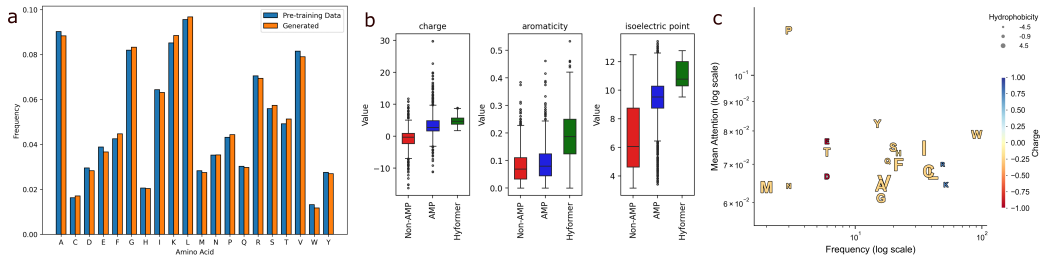}
  \vskip -0.2cm
  \caption{
    \textbf{(a)} Amino-acid distributions between the pre-training data and unconditionally generated~sequences. \textbf{(b)} Distributions of charge, aromaticity, and isoelectric point (pI) for: non-AMP, AMP and conditionally generated sequences. \textbf{(c)} Frequency of crossing an attention threshold (x-axis) vs.\ mean attention weight (y-axis) for distinct amino-acids, colored by charge and sized by hydrophobicity.
    }
  \label{fig:hyformer_AMP}
\end{figure*} 
\vspace{-2ex}

\newpage
\section*{Acknowledgments}
\begin{wrapfigure}{r}{0.33\linewidth}
    \centering
    \vspace{-5ex}
    \includegraphics[width=\linewidth]{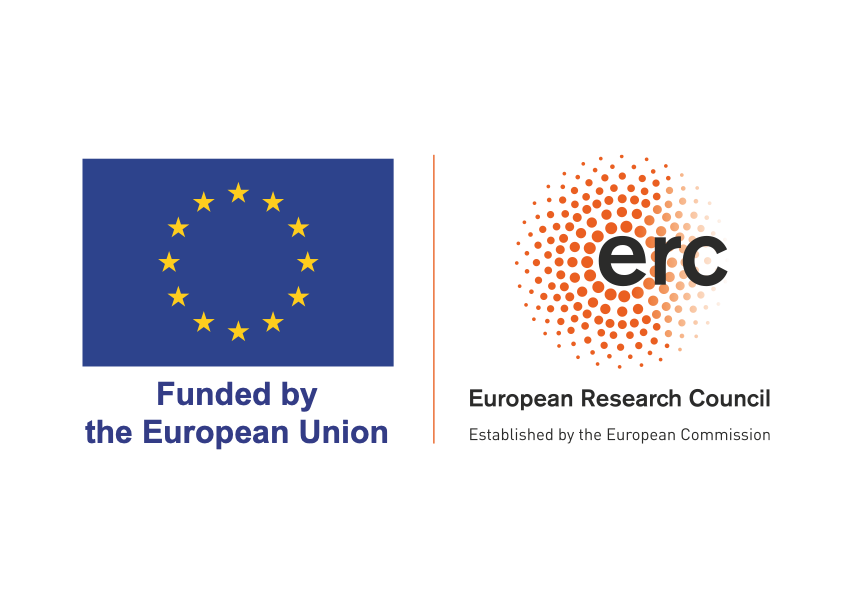}
    \vspace{-10ex}
\end{wrapfigure}

This project has received funding from the European Research Council (ERC) under the European Funding Union’s Horizon 2020 research and innovation programme (grant agreement No 810115 – DOG-AMP).

We thank Hassan Akell for insightful discussions and careful review of the theoretical part of this paper. Their feedback substantially improved the clarity, rigor, and presentation of the theoretical analysis.

\paragraph{Conflict of interest} Projects at Ewa Szczurek Lab at the University of Warsaw are co-funded by Merck Healthcare GmbH.

\bibliography{main}
\bibliographystyle{tmlr}

\appendix
\newpage

\section{Impact Statement}\label{appendix:impact}
The goal of this this work is 
to improve the field of deep generative modeling and, potentially, drug design. 
An example of potential malicious use of our approach would be training a deep generative model for generating new toxic molecules.
However, the intention of this paper is to provide tools that will facilitate designing new potential medications.

\section{Extended Discussion}

\paragraph{Extended Novelty Statement}
The alternating self-attention scheme is closely related to prior multitask transformer work (e.g., \citet{dong2019unified}). However, to the best of our knowledge, \textsc{Hyformer} is the first model to employ an alternating attention scheme during both pre-training and fine-tuning, resulting in a joint model that unifies molecular generation and property prediction. In contrast, \citet{dong2019unified} apply alternating attention only during pre-training. Moreover, \textsc{Hyformer} explicitly combines reconstruction-based losses: LM and MLM, with a prediction loss (Eq.~10), while \citet{dong2019unified} rely exclusively on reconstruction-based losses, without incorporating any supervision based on labeled data. Together, these architectural and objective-level differences enable joint generative and predictive modeling, and distinguish \textsc{Hyformer} from prior alternating-attention-based models.

\paragraph{Extended Future Work} An interesting direction for future work is to cast molecular property prediction as a purely generative task. While such a formulation could further unify generation and prediction within a single modeling paradigm, it introduces nontrivial challenges, most notably the principled tokenization and representation of continuous molecular properties. Moreover, the proposed framework naturally extends to modeling molecular interactions. Since \textsc{Hyformer} natively supports multimodal inputs, small-molecule–protein interactions can be incorporated by conditioning the decoder on protein embeddings and augmenting training with additional pretraining objectives. Exploring such extensions to jointly model multiple molecular modalities represents a promising avenue for future research.

\section{Notation}\label{appendix:notation}

\begin{table}[!h]
    \begin{tabular}{cl}
        \textbf{Symbol} & \textbf{Meaning} \\ 
        $[N]$  & Set of integers $1, \ldots, N$ \\
        $\mathbf{A}$    & Matrix \\
        $\mathbf{A}^T$    & Transposed matrix $\mathbf{A}$ \\
        $\mathbf{A}_{i}$, $\mathbf{A}_{ij}$, $\mathbf{A}^{ij}$    & Matrix indexed for some purpose \\
        $(\mathbf{A})_{i}, \mathbf{A}[i], A_{i} $    & The $i$-th row of matrix $\mathbf{A}$ \\
        $(\mathbf{A})_{ij}, \mathbf{A}[i, j], A_{ij} $    & The $i$-th, $j$-th entry of matrix $\mathbf{A}$ \\
        $\mathbf{a}$    & Vector (column-vector) \\
        $\mathbf{a}_{i}$, $\mathbf{a}_{ij}$, $\mathbf{a}^{ij}$    & Vector indexed for some purpose \\
        $(\mathbf{a})_{i}, \mathbf{a}[i], a_{i} $ & The $i$-th entry of vector $\mathbf{a}$ \\
        $a$ & Scalar \\ 
        $\X$   & input space, i.e.\ the space of all possible inputs, data examples \\
        $\Y$   & target space i.e.\ the space of all possible property values \\
        $p(\x, y)$ & joint data distribution \\
        $p_\theta(\x, y)$ & joint model parametrized by parameters $\theta \in \Theta$ \\
        $p_\theta(y \mid \x)$ & predictive model parametrized by parameters $\theta \in \Theta$ \\
        $p_\theta(\x)$ & generative model parametrized by parameters $\theta \in \Theta$ \\
    \end{tabular}
    \label{tab:my_label}
\end{table}

\section{Proofs}\label{appendix:proofs}

\subsection{Gradient Interference}\label{appendix:proofs-gradient-interference}

\begin{lemma}
\label{lem:softmax-matrix-gradient}
Let $\mathbf{x}\in\mathbb{R}^{I}$ and define
\begin{equation*}
    a_i =
\softmax(\mathbf{x})_i
=
\frac{\exp{x_i}}{\sum_{k=1}^{I}\exp{ x_k}} \text{ , for } i=1, \dots, I.
\end{equation*}
The Jacobian of the softmax is given by
\[
\frac{\partial a_i}{\partial x_j}
   \;=\;
   a_i\bigl(\,\delta_{ij}-a_j\bigr),
   \qquad i,j=1,\dots,I,
\]
where $\delta_{ij}$ is the Kronecker delta, i.e., $\delta_{ij}=1$ if $i=j$ and $0$ otherwise.
\end{lemma}
\begin{proof}
Differentiate the quotient
$
a_i = \exp x_i \bigl/ \sum_{k}\exp x_k
$
using the product and chain rules~\citep{petersen2008matrix}.
\end{proof}

\begin{corollary}
\label{cor:causal-softmax-jacobian}
Let $\Queries,\Keys\in\mathbb{R}^{T\times d}$ and the attention score matrix $\mathbf{S}_\rightarrow$ with a causal mask ${\Mask_\rightarrow}$ be defined as
\[
    \mathbf{S}_\rightarrow 
    =
    \frac{\Queries\Keys^{T}}{\sqrt{d}}+{\Mask}_{\rightarrow} \text{\;, where \;}
  ({\Mask}_{\rightarrow})_{ij}
     =
     \begin{cases}
       0        & \text{, if } i \geq j \\
       -\infty  & \text{, if } i < j.
     \end{cases}
\]
For a fixed row index $t \in [T]$, define the attention score row-vector 
$\mathbf{s}_t = (\mathbf{S})_t \in\mathbb{R}^{T}$ and the corresponding 
row-wise softmax output as  
$\mathbf{a}_{t}=\softmax(\mathbf{s}_t)\in\mathbb{R}^{T}$. The Jacobian of the softmax output $\mathbf{a}_{t}$ with respect to masked attention score $\mathbf{s}_t$ is
given by 
\[
\frac{\partial (\mathbf{a}_t)_{i}}{\partial (\mathbf{s}_t)_j} =
     (\mathbf{a}_{t})_{i} \bigl(\,\delta_{ij}-(\mathbf{a}_{t})_{j}\bigr).
\]
Hence, if $i <  t$ or $j < t$, while $i\neq j$, then $\frac{\partial (\mathbf{a}_t)_{i}}{\partial (\mathbf{s}_t)_j} = 0$.
\end{corollary}
\begin{proof}
Lemma~\ref{lem:softmax-matrix-gradient} gives the derivative of the softmax. As the causal mask sets $(\mathbf{s}_t)_j=-\infty$ for every
$j<t$, the corresponding probabilities satisfy $(\mathbf{a}_{t})_{j}=0$. 
\end{proof}

\section{Benchmark Task Definitions}

\subsection{Conditional Molecule Generation}

\textbf{Quantitative Estimate of Drug-likeness (QED).} A continuous metric of the drug-likeness of a molecule based on physicochemical properties such as molecular weight and hydrophobicity, with values ranging from 0 to 1. \citep{bickerton2012quantifying}

\textbf{Synthetic Accessibility (SA).} A continuous metric quantifying how difficult a molecule is to synthesize, derived from structural complexity, where lower values indicate easier synthesis. \citep{ertl2009estimation}

\textbf{Partition Coefficient (logP).} A continuous metric of molecular hydrophobicity, defined as the logarithm of the partition coefficient between octanol and water, where higher values denote greater affinity for lipophilic environments. \citep{wildman1999prediction}

Metric values calculated using rdkit 2023.09.2.

\subsection{Out-of-Distribution Molecular Property Prediction}

\textbf{DRD2-Hi.} Binary classification dataset of 8482 compounds with labels indicating dopamine receptor inhibition, with therapeutic relevance in schizophrenia and Parkinson’s disease; dataset obtained from ChEMBL30. \citep{mendez2019chembl}

\textbf{HIV-Hi.} Binary classification dataset of 41127 compounds from the Drug Therapeutics Program AIDS Antiviral Screen, with labels indicating the inhibition of HIV replication; dataset obtained from MoleculeNet. \citep{wu2018moleculenet}

\textbf{KDR-Hi.} Binary classification dataset with labels indicating VEGFR2 (vascular endothelial growth factor receptor 2) inhibition, a kinase target in cancer therapy, with training restricted to 500 compounds to simulate low-data regimes; dataset obtained from Chembl30. \citep{mendez2019chembl}

\textbf{Sol-Hi.} Binary classification dataset of 2173 compounds with labels indicating solubility; dataset obtained at Biogen. \citep{fang2023prospective}

For further dataset and train/test splitting details, see \citep{steshin2023lohipracticalmldrug}. Data accessed from \url{https://github.com/SteshinSS/lohi_neurips2023/tree/main/data/hi} [accessed 20.03.2023].

\subsection{Molecular Representation Learning and Property Prediction}

\textbf{ESOL.} Regression dataset containing water solubility measurements for 1128 compounds.

\textbf{FreeSolv.} Regression dataset containing experimentally measured hydration free energy values in water for 642 compounds.

\textbf{Lipophilicity.} Regression dataset containing experimentally measured octanol/water distribution coefficients (logD at pH 7.4), a key indicator of membrane permeability and solubility, for 4,200 compounds.

\textbf{BACE.} Binary classification dataset of 1513 compounds with experimentally determined qualitative binding results for a set of inhibitors of human $\beta$-secretase 1 (BACE-1).

\textbf{BBBP.} Binary classification dataset of 2039 compounds with binary labels indicating blood–brain barrier permeability.

\textbf{ClinTox.} Multitask classification dataset of 1478 compounds with labels indicating whether a compound is (i) FDA-approved and/or (ii) failed clinical trials due to toxicity reasons.

\textbf{HIV.} Binary classification dataset of 41127 compounds from the Drug Therapeutics Program AIDS Antiviral Screen, measuring inhibition of HIV replication.

\textbf{Tox21.} Multitask classification dataset of 7831 compounds with qualitative toxicity measurements across 12 biological targets, including nuclear receptors and stress response pathways.

\textbf{ToxCast.} Multitask classification dataset of 8575 compounds with qualitative toxicity results across over 600 in vitro assays, derived from high-throughput screening.

\textbf{SIDER.} Multitask classification dataset of 1427 approved drugs, with side effects grouped into 27 system organ classes according to MedDRA classifications, capturing adverse drug reactions across organ systems.

For further details, see Table 1 in \citet{wu2018moleculenet}. To ensure comparability with Uni-Mol \citep{zhou2023unimol}, we accessed data from \url{https://bioos-hermite-beijing.tos-cn-beijing.volces.com/unimol_data/finetune/molecular_property_prediction.tar.gz} [accessed 20.03.2023].

\section{Benchmark Metric Definitions}

\textbf{MAD.} Mean Absolute Deviation between predicted and target property values; lower is better.

\textbf{SD.} Standard Deviation of generated property values from the target; lower is better.

\textbf{Validity.} Fraction of syntactically valid molecules generated by the model; higher is better.

\textbf{Uniqueness.} Fraction of unique molecules among generated samples; higher is better.

\textbf{Novelty.} Fraction of generated molecules not present in the training set; higher is better.

\textbf{KL Div.\ Score.} Score based on the Kullback–Leibler Divergence between various descriptor distributions of generated and training molecules; values normalized in the range [0, 1]; higher values indicate a closer match between descriptor distributions between generated and training molecules. \citep{brown2019guacamol}

\textbf{FCD Score.} Score based on the Fréchet ChemNet Distance between the generated and reference (training) molecule embedding distributions, calculated in ChemNet feature space; values normalized in the range [0, 1]; higher values indicate closer resemblance of the generated to reference molecules. \citep{brown2019guacamol}

\textbf{Perplexity.} Exponentiated negative log-likelihood of a sequence, with the log-likelihood being calculated per token, using ProGen2-medium \citep{Torres2025.01.31.636003}; lower values indicate greater model-based plausibility of the generated peptides.

\textbf{Diversity.} Average pairwise Levenshtein distance between the generated sequences; higher values indicate greater diversity of the generated samples. For details, see Eq. 6 in \citet{kim2021survey}, where Hyformer replaces Soergel with Levenshtein distance.

\textbf{Fitness.} A measure quantifying to what extent a peptide forms a stable, amphipathic $\alpha$-helix, computed according. \citep{Li2024-pa}

\textbf{HydrAMP MIC.} The probability of a peptide being active against E.Coli bacteria strain predicted with HydrAMP. \citep{Szymczak2022.01.27.478054}

\textbf{AMPlify.} The probability of a peptide being antimicrobial predicted with AMPlify. \citep{Li2022-us}

\textbf{amPEPy.} The probability of a peptide being antimicrobial predicted with amPEPy. \citep{10.1093/bioinformatics/btaa917}

\section{Pre-training Details}\label{appendix:pre-training}

We implement $\HyF$ using a LLAMA backbone~\citep{touvron2023llama}. Depending on the size of the pretraining dataset, we scale $\HyF$ to 8.7M parameters for GuacaMol\footnote{Data accessed from \url{https://figshare.com/projects/GuacaMol/56639} on 20.03.2025.} and 50M parameters for the UniMol\footnote{Data accessed from \url{https://bioos-hermite-beijing.tos-cn-beijing.volces.com/unimol_data/finetune/molecular_property_prediction.tar.gz} on 20.03.2023.} and peptide datasets
\footnote{
Data accessed from \url{
https://app.peptipedia.cl/
}, \url{
https://www.uniprot.org/uniprotkb?query=\%28length%
}, \url{
https://ampsphere.big-data-biology.org/downloads
} and \url{ https://drive.google.com/drive/folders/1krim1ugqNDmgmHZCFSOvmynWxCSzyOto
} on 17.04.2025 with train/test set constructed using standard scikit’s train/test splitting and random seed 44.
}. These configurations align model capacity with dataset size and ensure a fair comparison with prior work: the 8.7M model is comparable to MolGPT~\citep{bagal2022molgpt}, while the 50M variant matches the scale of Uni-Mol~\citep{zhou2023unimol} and Graph2Seq~\citep{gao2024graph}. For GuacaMol, we apply 2× data augmentation using non-canonical SMILES enumeration~\citep{bjerrum2017smilesenumerationdataaugmentation, arus2019randomized} to increase molecular diversity. All models are pretrained using pre-computed molecular descriptors~\citep{yang2019analyzing}. The balancing of the tasks $(p_{\LMToken}, p_{\MLMToken}, p_{\PredToken})$ is set to (0.90, 0.05, 0.05) and (0.80, 0.10, 0.10), respectively.

We use SMILES~\citep{weiningerSMILESChemicalLanguage1988} or amino acid sequences as molecular representations across all experiments. For tokenization, we adopt an extended character-level tokenizer for SMILES, based on \citet{schwaller_probst_vaucher_nair_kreutter_laino_reymond_2020}, and use the ESM-2 tokenizer~\citep{esm2} for peptides.

We pre-train $\HyF$ using a batch size of 1024 for up to 50K or 250K iterations, depending on model size. Training is performed with the AdamW optimizer ($\beta_1 = 0.9$, $\beta_2 = 0.95$, $\epsilon = 1 \times 10^{-5}$, weight decay = $1 \times 10^{-1}$), using a peak learning rate of $6 \times 10^{-4}$ with cosine decay and 5000 warm-up steps. We use gradient clipping with a maximum norm of 1.0. All input sequences are padded to a fixed length of 128 tokens. Training is conducted using bfloat16 precision on a single NVIDIA H100 80GB HBM3 GPU. 

\begin{table}[h]
\caption{Architectural details of $\HyF$.}\label{table:architectural-details}
\vskip 0.15in
\begin{center}
\begin{scriptsize}
    \begin{sc}
\begin{tabular}{lccccc}
\toprule
 Num. param. & Embed.\ dim & Hidden dim & \#Layers & \# Att.\ Heads \\
\midrule
8.7M & 256 & 1024 & 8 & 8 \\
50M & 512 & 2048 & 12 & 8 \\
\bottomrule
\end{tabular}
\end{sc}
\end{scriptsize}
\end{center}
\vskip -0.1in
\end{table}

\section{Experimental Details}\label{appendix:experimental-details}

All fine-tuning and inference is conducted using float32 precision on a single NVIDIA V100 32GB GPU.

\subsection{Conditional Molecule Generation}\label{appendix:conditional-molecule-generation-task}

We jointly fine-tune $\HyF$, pretrained on GuacaMol dataset, for 10 epochs with a batch size of 256. The peak learning rate is selected from the set $\{1\mathrm{e}{-4}, 2\mathrm{e}{-4}, 3\mathrm{e}{-4}, 4\mathrm{e}{-4}, 5\mathrm{e}{-4}, 5\mathrm{e}{-4}, 6\mathrm{e}{-4}\}$, based on root mean squared error (RMSE) with respect to the target property. During fine-tuning, we set the task probability vector to $(p_{\LMToken}, p_{\PredToken}) = (0.5, 0.5)$ and do not perform hyperparameter search over this setting, as it yields satisfactory performance by default. For the non-joint variant of $\HyF$, we freeze the pretrained model and fine-tune only the prediction head. This avoids catastrophic forgetting of the generative capability when removing the generative loss during training. For each target property value, we sample 100K unique molecules, with a wall-clock time of 78 $\pm$ 1 seconds, and retain those passing a manually defined threshold, using multinomial top-$k$ sampling with $\tau=0.9$ and $k=10$. Note that reported SA scores are normalized, following \citep{gao2024graph}.

To further characterize the selectivity of conditional sampling and the calibration of the predictive heads, we report acceptance rates in the conditional molecule generation experiment in Table~\ref{tab:conditional_acceptance}. The results confirm that conditional sampling with \textsc{Hyformer} is highly selective across all target values.

{\color{blue}\begin{table}[h!]
\caption{Conditional generative performance on GuacaMol dataset across all targets. Best model is marked \textbf{bold}.}
\begin{center}
\begin{small}
\begin{sc}
\begin{adjustbox}{width=\columnwidth}
\begin{tabular}{ccclcccccccccc}
\toprule
& Pretrain & Joint & Metric & QED=0.5 & QED=0.7 & QED=0.9 & SA=0.7 & SA=0.8 & SA=0.9 & logP=0.0 & logP=2.0 & logP=4.0 & Avg. \\
\midrule
\parbox[t]{2mm}{\multirow{3}{*}{\scriptsize\rotatebox[origin=c]{90}{MolGPT}}} & \parbox[t]{2mm}{\multirow{3}{*}{\ding{55}}} & \parbox[t]{2mm}{\multirow{3}{*}{\ding{55}}}
& MAD $\downarrow$ & 0.081 & 0.082 & 0.097 & 0.024 & 0.019 & 0.013 & 0.304 & 0.239 & 0.286 & 0.127 \\
& & & SD $\downarrow$  & 0.065 & 0.066 & 0.092 & 0.022 & 0.016 & 0.013 & 0.295 & 0.232 & 0.258 & 0.118 \\
& & & Validity $\uparrow$ & 0.985 & 0.985 & 0.984 & 0.975 & 0.988 & 0.995 & 0.982 & 0.983 & 0.982 & 0.984 \\
\midrule
\parbox[t]{2mm}{\multirow{6}{*}{\scriptsize\rotatebox[origin=c]{90}{GraphGPT-1W-C}}} & \parbox[t]{2mm}{\multirow{3}{*}{\ding{55}}} & \parbox[t]{2mm}{\multirow{3}{*}{\ding{55}}} & MAD $\downarrow$ & 0.041 & 0.031 & 0.077 & 0.012 & 0.028 & 0.031 & 0.103 & 0.189 & 0.201 & 0.079 \\
& & & SD $\downarrow$  & 0.079 & 0.077 & 0.121 & 0.055 & 0.062 & 0.070 & 0.460 & 0.656 & 0.485 & 0.229 \\
& & & Validity $\uparrow$ & 0.988 & 0.995 & 0.991 & 0.995 & 0.991 & 0.998 & 0.980 & \textbf{0.992} & 0.991 & 0.991 \\
\cmidrule(lr){2-14}
& \parbox[t]{2mm}{\multirow{3}{*}{\ding{51}}} & \parbox[t]{2mm}{\multirow{3}{*}{\ding{55}}} & MAD $\downarrow$ & 0.032 & 0.033 & 0.051 & \textbf{0.002} & 0.009 & 0.022 & \textbf{0.017} & 0.190 & 0.268 & 0.069 \\
&&& SD $\downarrow$  & 0.080 & 0.075 & 0.090 & 0.042 & 0.037 & 0.062 & 0.463 & 0.701 & 0.796 & 0.261 \\
&&& Validity $\uparrow$ & \textbf{0.996} & \textbf{0.998} & \textbf{0.999} & \textbf{0.995} & \textbf{0.999} & 0.996 & 0.994 & 0.990 & 0.992 & \textbf{0.995} \\
\midrule
\parbox[t]{2mm}{\multirow{6}{*}{\scriptsize\rotatebox[origin=c]{90}{Hyformer}}} & \parbox[t]{2mm}{\multirow{3}{*}{\ding{51}}} & \parbox[t]{2mm}{\multirow{3}{*}{\ding{55}}} & 
    MAD $\downarrow$ & 0.035 (0.000) & 0.032 (0.001) & 0.027 (0.007) & 0.020 (0.001) & 0.016 (0.000) & 0.009 (0.001) & 0.131 (0.012) & 0.135 (0.007) & 0.127 (0.011) & 0.059 (0.004) \\
&&& SD $\downarrow$  & 0.049 (0.000) & 0.046 (0.002) & 0.039 (0.010) & 0.027 (0.002) & 0.021 (0.000) & 0.012 (0.001) & 0.162 (0.015) & 0.174 (0.010) & 0.175 (0.016) & 0.078 (0.006) \\
&&& Validity $\uparrow$ & 0.993 (0.003) & 0.993 (0.003) & 0.993 (0.004) & 0.986 (0.009) & 0.985 (0.001) & \textbf{0.999 (0.002)} & 0.978 (0.031) & 0.983 (0.004) & \textbf{0.995 (0.007)} & 0.989 (0.007) \\
\cmidrule(lr){2-14}
& \parbox[t]{2mm}{\multirow{3}{*}{\ding{51}}} & \parbox[t]{2mm}{\multirow{3}{*}{\ding{51}}} 
    & MAD $\downarrow$ & \textbf{0.010 (0.001)} & \textbf{0.009 (0.000)} & \textbf{0.006 (0.001)} & 0.008 (0.001) & \textbf{0.005 (0.000)} & \textbf{0.001 (0.000)} & 0.033 (0.005) & \textbf{0.044 (0.001)} & \textbf{0.046 (0.001)} & \textbf{0.018 (0.001)} \\
  &&& SD $\downarrow$ & \textbf{0.018 (0.002)} & \textbf{0.018 (0.002)} & \textbf{0.010 (0.003)} & \textbf{0.015 (0.003)} & \textbf{0.009 (0.002)} & \textbf{0.004 (0.000)} & \textbf{0.037 (0.007)} & \textbf{0.057 (0.003)} & \textbf{0.059 (0.001)} & \textbf{0.025 (0.003)} \\
  &&& Validity $\uparrow$ & 0.983 (0.009) & 0.990 (0.006) & 0.996 (0.005) & 0.976 (0.005) & 0.981 (0.002) & \textbf{0.999 (0.002)} & \textbf{1.000 (0.000)} & 0.991 (0.007) & 0.971 (0.011) & 0.987 (0.005) \\
\bottomrule
\end{tabular}
\end{adjustbox}
\end{sc}
\end{small}
\end{center}
\end{table}}

\begin{algorithm}[h]
	\caption{Conditional sampling with $\HyF$}
    \label{algorithm:jointformer-sampling}
    \let\AND\relax
	\begin{algorithmic}[1]
		\REQUIRE Number of examples to sample $K$, batch size $B$, condition $Y$, model parameters $\theta$.
        \STATE $\mathcal{D}_{sampled} = \emptyset$
        \WHILE{$|\mathcal{D}_{sampled}| < K$}
        \STATE Sample $B$ many examples $(\x, y) \sim \joint$
        \STATE Accept examples $\mathcal{D}_{batch} = \{(\x, y) \mid y \in Y\}$ 
        \STATE Append dataset $\mathcal{D}_{sampled} = \mathcal{D}_{sampled} \cup \mathcal{D}_{batch}$  
        \ENDWHILE
	\end{algorithmic} 
\end{algorithm}

\begin{table}[h]
\centering
\caption{Number of accepted samples per 100{,}000 generated molecules in the conditional generation experiment. Mean and standard deviation across three random seeds. The average (Avg.) is computed over all target~values.}
\label{tab:conditional_acceptance}
\begin{scriptsize}
\begin{adjustbox}{width=\columnwidth}
\begin{tabular}{lcccccccccc}
\toprule
Model 
& QED=0.5 
& QED=0.7 
& QED=0.9 
& SA=0.7 
& SA=0.8 
& SA=0.9 
& logP=0.0 
& logP=2.0 
& logP=4.0 
& Avg. \\
\midrule
\textsc{Hyformer} (no-joint)
& 315 (7)
& 367 (11)
& 162 (11)
& 144 (5)
& 448 (15)
& 229 (23)
& 14 (3)
& 76 (7)
& 70 (2)
& 203 (81) \\
\textsc{Hyformer} (joint)
& 295 (19)
& 330 (4)
& 176 (17)
& 140 (10)
& 426 (13)
& 254 (7)
& 11 (4)
& 70 (7)
& 67 (4)
& 197 (71) \\
\bottomrule
\end{tabular}
\end{adjustbox}
\end{scriptsize}
\end{table}

\subsection{Out-of-Distribution Molecular Property Prediction Task}\label{appendix:ood-property-prediction-task}

We use $\HyF$ pre-trained on UniMol dataset and perform a grid search over hyperparameters, as detailed in Table~\ref{table:hparams-grid-search-ood}, with end-to-end joint fine-tuning, with early stopping triggered if the validation loss does not improve for 5 consecutive epochs. Results in Table~\ref{table:hi-task} are reported from \citep{steshin2023lohipracticalmldrug}.

\begin{table}[h]
\caption{Hyperparameter ranges for the grid search hyperparameter optimization on out-of-distribution molecular property prediction task.}\label{table:hparams-grid-search-ood}
\vskip 0.15in
\begin{center}
\begin{scriptsize}
\begin{sc}
\begin{tabular}{lc}
\toprule
Hyperparameter & Search Range \\
\midrule
Max Epochs & \{20, 50, 100\} \\
Batch Size & \{64, 128, 256\} \\
Learning Rate & [1e-5, 6e-4] \\
Weight Decay & [1e-2, 1e-1] \\
Pooler Dropout & [0.0, 0.2] \\
Learning Rate Decay & \{True, False\} \\
$(p_{\LMToken}, p_{\PredToken})$ & \{(0.0, 1.0), (0.1, 0.9)\} \\
\bottomrule
\end{tabular}
\end{sc}
\end{scriptsize}
\end{center}
\vskip -0.1in
\end{table}

\subsection{Molecular Representation Learning Task}\label{appendix:representation-learning-task}

For KNN probe, we use the Euclidean norm to pick K most similar molecules. For each dataset, we search the parameter K in the set $\{1,3,5,100,300,500,1000,3000,5000\}$ and pick K with the best performance on the validation split. For linear probe, we report the results of linear probe with L2 regularization added. If the validation loss between the epochs does not decrease by more than $0.0001$ for $10$ consecutive epochs, we terminate the training process early. All results in Table~\ref{table:molecular-representation-learning-task} are ours.

\subsection{Molecule Generation Task}\label{appendix:molecule-generation-task}

For generation, we use $\HyF$ pre-trained on GuacaMol and sample using multinomial top-$k$ sampling, with~$k=10$ and varying temperature $\tau=\{0.9, 1.0, 1.1\}$.

In Table~\ref{table:guacamol-distribution-learning-benchmark}, baseline results for JTVAE and MAGNeT are reported from \citep{hetzel2023magnet}, for MoLeR and MiCaM from \citep{geng2023novo}, for VAE, LSTM from \citep{brown2019guacamol}, for MolGPT from \citep{bagal2022molgpt}. 

\subsection{Molecular Property Prediction Task}\label{appendix:property-prediction-task}

We use $\HyF$ pre-trained on UniMol dataset and perform a grid search over hyperparameters, as detailed in Table~\ref{table:hparams-grid-search}, with end-to-end predictive fine-tuning run for a maximum of 20 epochs, with early stopping triggered if the validation loss does not improve for 5 consecutive epochs. Results in Table~\ref{table:molecular-property-prediction-task} are reported from \citep{zhou2023unimol, gao2024graph}.

\begin{table}[h]
\caption{Hyperparameter ranges for the grid search hyperparameter optimization on molecular property prediction task.}\label{table:hparams-grid-search}
\vskip 0.15in
\begin{center}
\begin{scriptsize}
\begin{sc}
\begin{tabular}{lc}
\toprule
Hyperparameter & Search Range \\
\midrule
Batch Size & \{16, 64, 128, 256\} \\
Learning Rate & [1e-5, 1e-3] \\
Weight Decay & [1e-2, 3e-1] \\
Pooler Dropout & [0.0, 0.2] \\
Learning Rate Decay & \{True, False\} \\
\bottomrule
\end{tabular}
\end{sc}
\end{scriptsize}
\end{center}
\vskip -0.1in
\end{table}

\subsection{Antimicrobial Peptide Design}\label{appendix:antimicrobial-peptide-design}

\paragraph{Dataset} We construct a general-purpose peptide dataset and an AMP-specific dataset. For the general purpose dataset, we collect 3459247 peptide sequences with length 8-50 from the combined Peptipedia \citep{Cabas-Mora2024.07.11.603053} and UniProt \citep{10.1093/nar/gkae1010} datasets and apply CDHIT filtering with a similarity threshold of 90\%. For the AMP-specific dataset, we collect 1056321 sequences from combining the Peptipedia \citep{Cabas-Mora2024.07.11.603053}, filtered with Antigram (-), Antigram (+), Antibacterial and Antimicrobial keywords, Uniprot with the keywords antimicrobial and AMPSphere \citep{santos_junior_2022_6511404}, and  applying CDHIT filtering with a similarity threshold of 90\%.

\paragraph{Pre-trainig} We pre-train $\HyF$ in a two-stage manner, by first training on the general-purpose, followed by training on the AMP specific dataset with peak learning rate equal to $4\mathrm{e}{-4}$. All additional details follow Appendix~\ref{appendix:pre-training}.

\paragraph{Fine-tuning} We fine-tune $\HyF$ for a maximum of 10 epochs, with batch size 64, peak learning rate $5\mathrm{e}{-5}$ and early stopping, with task probabilities $(p_{\LMToken}, p_{\PredToken})$ equal to (0.6, 0.4). Additionally, we freeze the first four layers of the model. 

\paragraph{Conditional Sampling} For the antimicrobial peptide design experiment, we unconditionally sample around 6.7M sequences and accept 50K of them, which results in an acceptance rate of around 0.7\%. Note that our sampling procedure intentionally prioritizes high selectivity over throughput. All peptides have a maximum length of 50 AAs.

\section{Additional Experiments}\label{appendix:additional-experiments}

\subsection{Unconditional Molecule Generation on MOSES benchmark}

\begin{table}
\vspace{-2ex}
\caption{Unconditional generative performance on MOSES benchmark.
The best model in each category is marked \textbf{bold}.}
\label{table:moses-distribution-learning-benchmark}
\begin{center}
\begin{scriptsize}    
\begin{sc}
\begin{tabular}{lccccc}
\toprule
Model & Validity $\uparrow$ & Unique $\uparrow$ & Novelty $\uparrow$ & IntDiv1 $\uparrow$ & IntDiv2 $\uparrow$ \\
\midrule
\multicolumn{6}{l}{\textit{Unconditional}} \\
HMM            & 0.076 & 0.567 & \textbf{0.999} & 0.847 & 0.810 \\
NGram          & 0.238 & 0.922 & 0.969 & \textbf{0.874} & 0.864 \\
Combinatorial  & \textbf{1.000} & 0.991 & 0.988 & 0.873 & \textbf{0.867} \\
CharRNN        & 0.975 & 0.999 & 0.842 & 0.856 & 0.850 \\
VAE            & 0.977 & 0.998 & 0.695 & 0.856 & 0.850 \\
AEE            & 0.937 & 0.997 & 0.793 & 0.856 & 0.850 \\
LatentGAN      & 0.897 & 0.997 & 0.949 & 0.857 & 0.850 \\
JT-VAE         & \textbf{1.000} & 0.999 & 0.914 & 0.855 & 0.849 \\
MolGPT         & 0.994 & \textbf{1.000} & 0.797 & 0.857 & 0.851 \\
$\textsc{Hyformer}_{\tau=0.9}$ \hh & \hh 0.996 & \hh \textbf{1.000} & \hh 0.701 & \hh 0.851 & \hh 0.845 \\
$\textsc{Hyformer}_{\tau=1.0}$ \hh & \hh 0.991 & \hh \textbf{1.000} & \hh 0.749 & \hh 0.856 & \hh 0.850 \\
$\textsc{Hyformer}_{\tau=1.1}$ \hh & \hh 0.986 & \hh \textbf{1.000} & \hh 0.791 & \hh 0.861 & \hh 0.855 \\
\midrule
\multicolumn{6}{l}{\textit{Few-Shot}} \\
$\text{GraphGPT-1W}_{s=0.25}$     & \textbf{0.995} & 0.995 & 0.255 & 0.854 & 0.850 \\
$\text{GraphGPT-1W}_{s=0.5}$      & 0.993 & 0.996 & 0.334 & 0.856 & 0.848 \\
$\text{GraphGPT-1W}_{s=1.0}$     & 0.978 & 0.997 & 0.871 & \textbf{0.860} & \textbf{0.857} \\
$\text{GraphGPT-1W}_{s=2.0}$      & 0.972 & \textbf{1.000} & \textbf{1.000} & 0.850 & 0.847 \\
\bottomrule
\end{tabular}
\end{sc}
\end{scriptsize}
\end{center}
\vspace{-2ex}
\end{table}

To additionally evaluate the unconditional generative performance of $\HyF$, we perform an evaluation on the MOSES benchmark. Analogously to unconditional molecule generation in Section~\ref{section:unconditional-molecule-generation}, we scale $\HyF$ to 8.5M parameters and follow all the training details in Appendix~\ref{appendix:pre-training} for GuacaMol dataset. We compare $\HyF$, across various sampling temperatures $\tau$, to baseline unconditional and few-shot generative models, as reported in \citep{gao2024graph}.

$\HyF$ successfully generates valid, unique, novel and diverse molecules, rivaling other unconditional and few-shot generative models.

\subsection{Qualitative Evaluation of Generated Molecules}

To investigate the effect of sampling temperature on the structural diversity and chemical quality of generated molecules, we show molecules sampled in the unconditional generation task (Section~\ref{section:unconditional-molecule-generation}), at temperatures $\tau  = 0.9$, $1.0$, and $1.1$. For each sampled molecule, we additionally report four chemical properties: molecular partition coefficient (LogP), topological polar surface area (TPSA), quantitative estimate of drug-likeness (QED) and molecular weight (MW). At $\tau  = 0.9$, the model generates drug-like molecules, with the majority exhibiting QED $\geq$ 0.7 and MW $<$ 500~g/mol (Fig.~\ref{fig:T09generatedmolecules}). At $\tau  = 1.0$, the sampling process yields molecules with greater structural diversity (Fig.~\ref{fig:T1.0generatedmolecules}). Despite the increased exploration of chemical space, some molecules exhibit lower QED values. At $\tau  = 1.1$, the model produces molecules with less common substituent patterns. Some of these structures exceed traditional drug-likeness thresholds, such as MW $>$ 500~g/mol or LogP $>$ 5, according to Lipinski’s Rule of Five (Fig.~\ref{fig:T1.1generatedmolecules}). Additionally, we investigate molecules generated in the conditional generation task in Section~\ref{subsubsection:experiment-conditional-molecule-generation} (Figure~\ref{fig:generatedmoleculesconditionedQED}, \ref{fig:generatedmoleculesconditionedSA}~and~\ref{fig:generatedmoleculesconditionedlogP}).

\begin{figure}[!h]
    \centering
    \includegraphics[width=0.75\columnwidth]{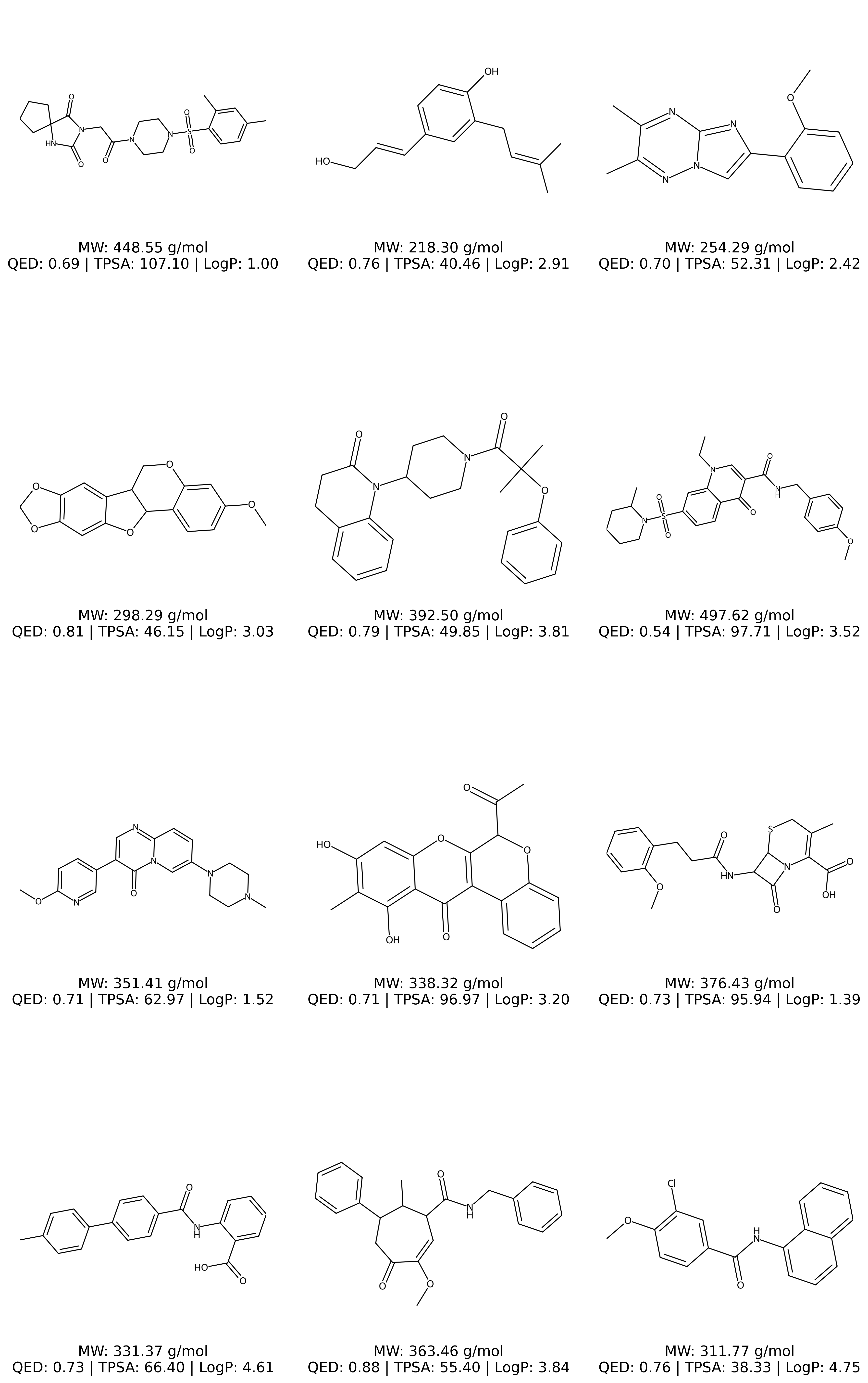}
    \caption{Structures of the twelve generated molecules with Hyformer when the sampling temperature is 0.9, visualized using RDKit, together with their properties.
    }
    \label{fig:T09generatedmolecules}
\end{figure}

\begin{figure}[!h]
    \centering
    \includegraphics[width=0.75\columnwidth]{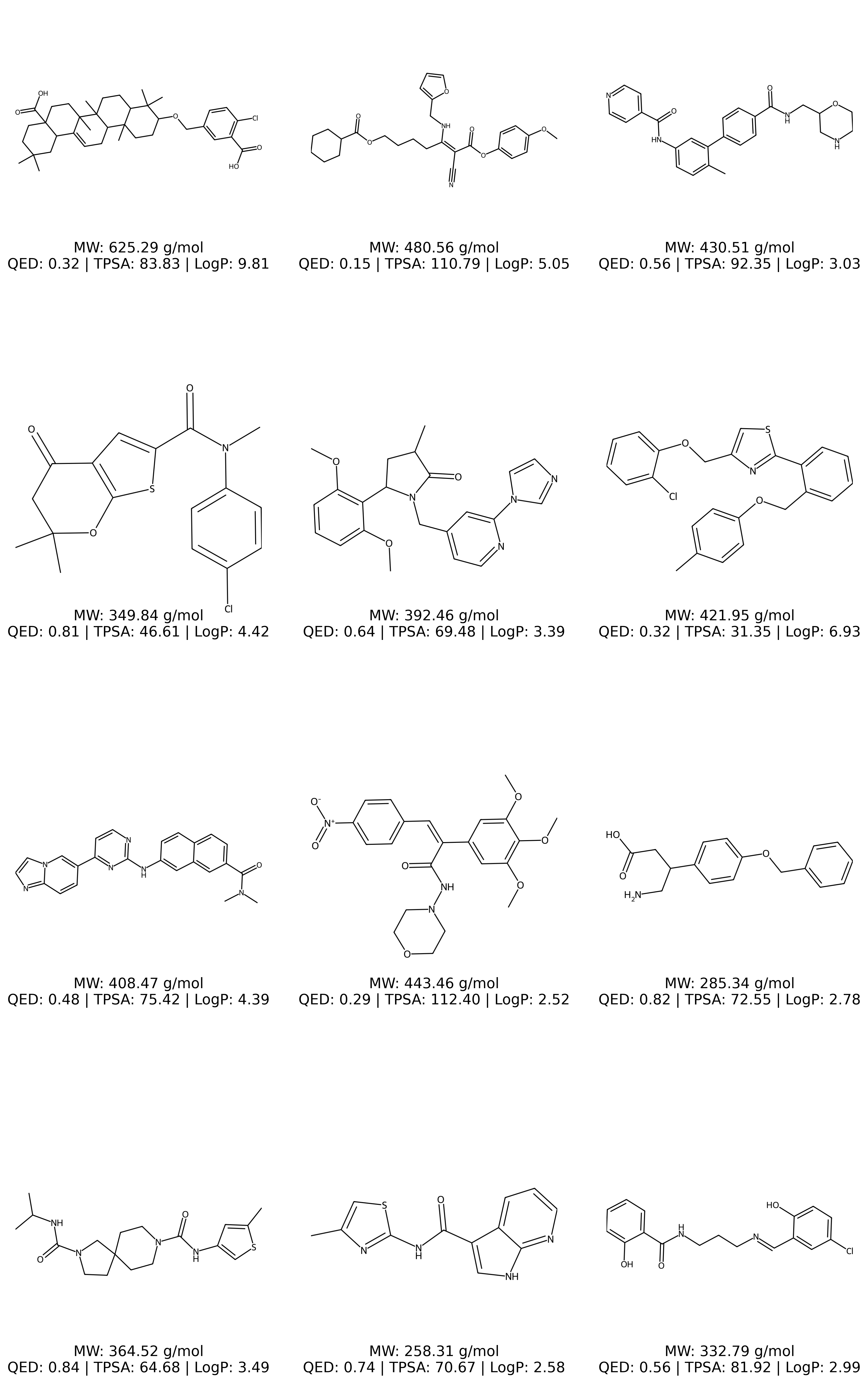}
    \caption{Structures of the twelve generated molecules with Hyformer when the sampling temperature is 1.0, visualized using RDKit, together with their properties.
    }
    \label{fig:T1.0generatedmolecules}
\end{figure}

\begin{figure}[!h]
    \centering
    \includegraphics[width=0.75\columnwidth]{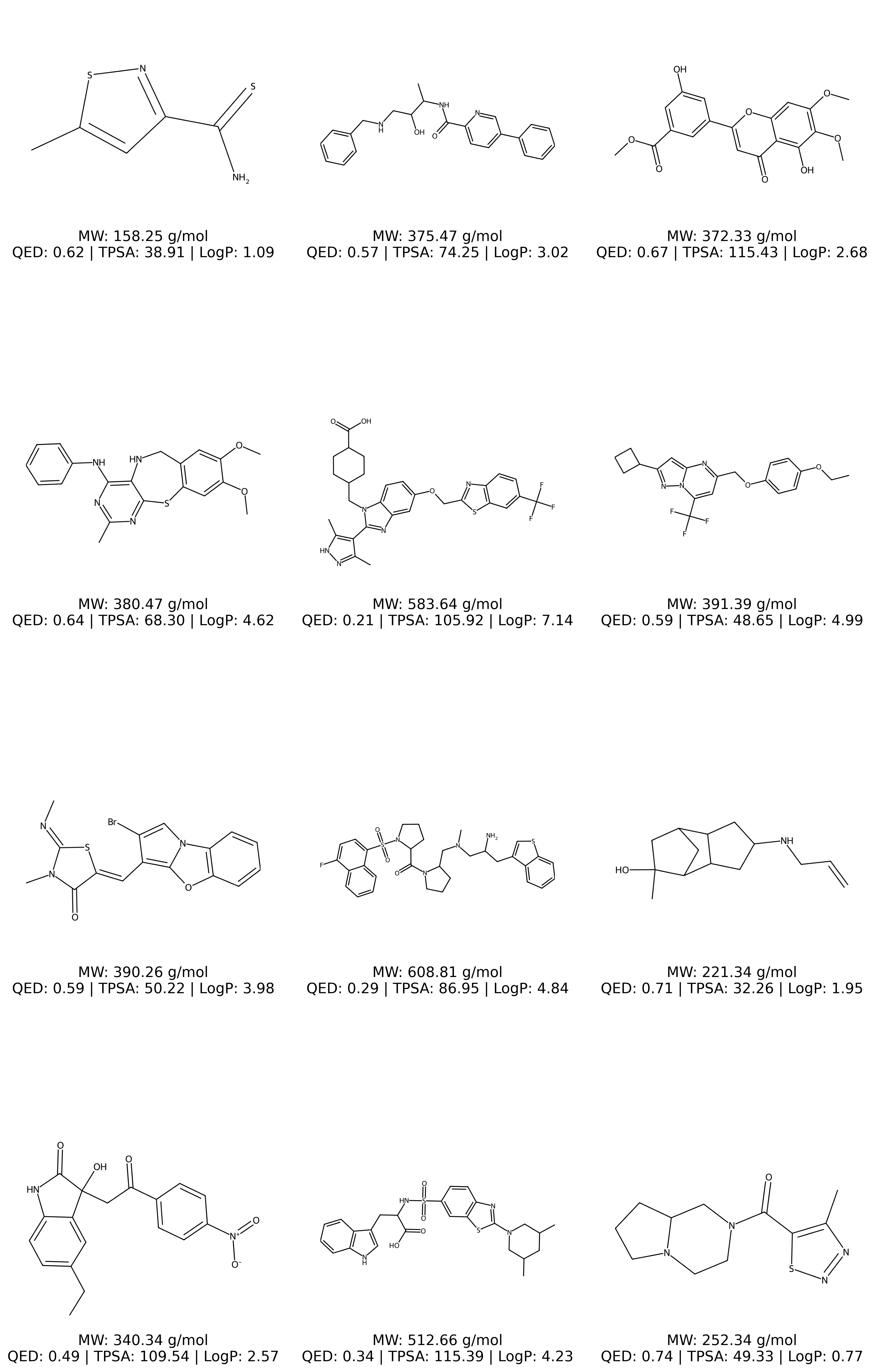}
    \caption{Structures of the twelve generated molecules with Hyformer when the sampling temperature is 1.1, visualized using RDKit, together with their properties.
    }
    \label{fig:T1.1generatedmolecules}
\end{figure}

\begin{figure}[!h]
    \centering
    \includegraphics[width=\columnwidth]{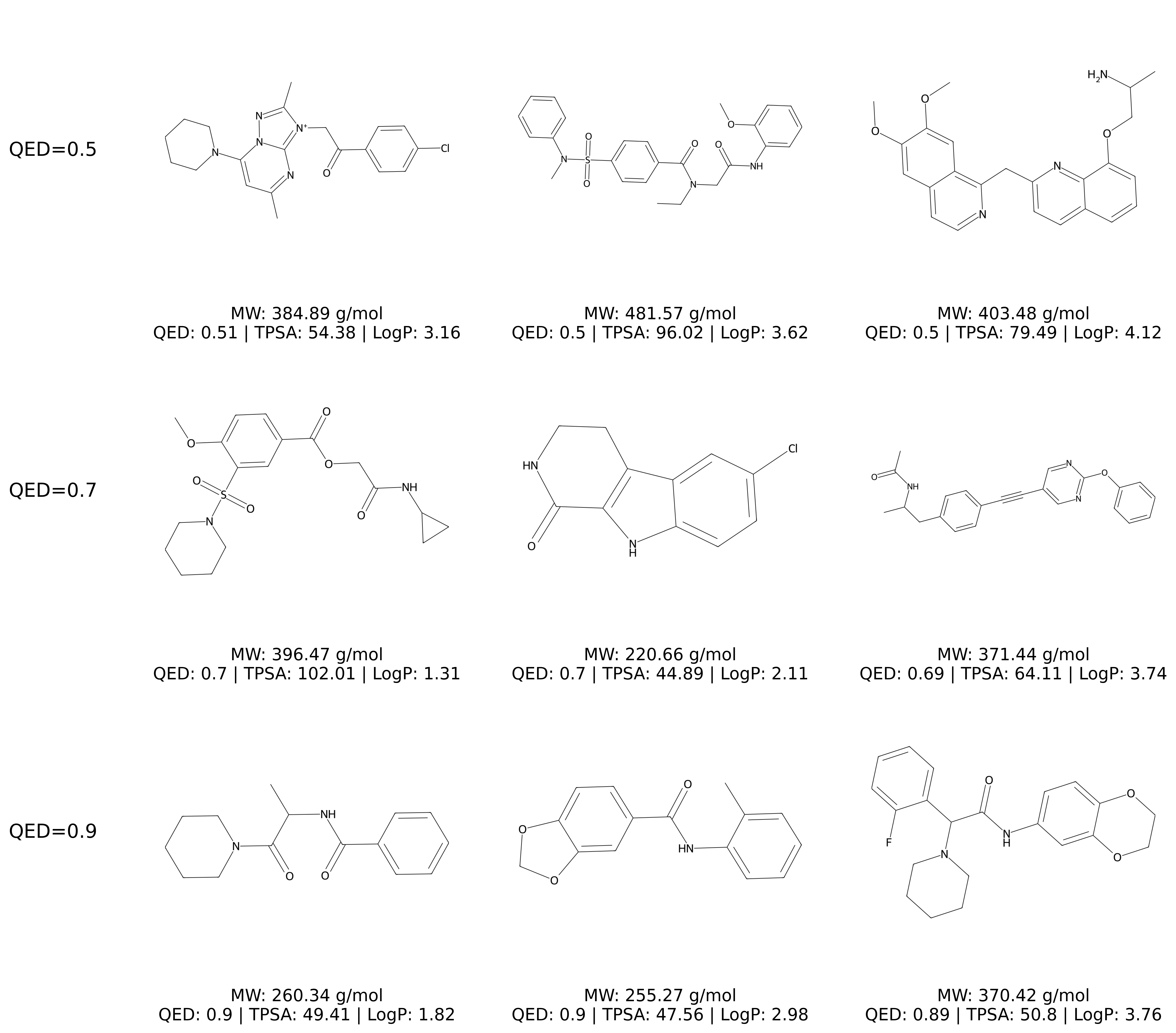}
    \caption{Structures of molecules generated by Hyformer conditioned on QED values, visualized using RDKit, along with their chemical properties.
    }
    \label{fig:generatedmoleculesconditionedQED}
\end{figure}

\begin{figure}[!h]
   \centering
   \includegraphics[width=\columnwidth]{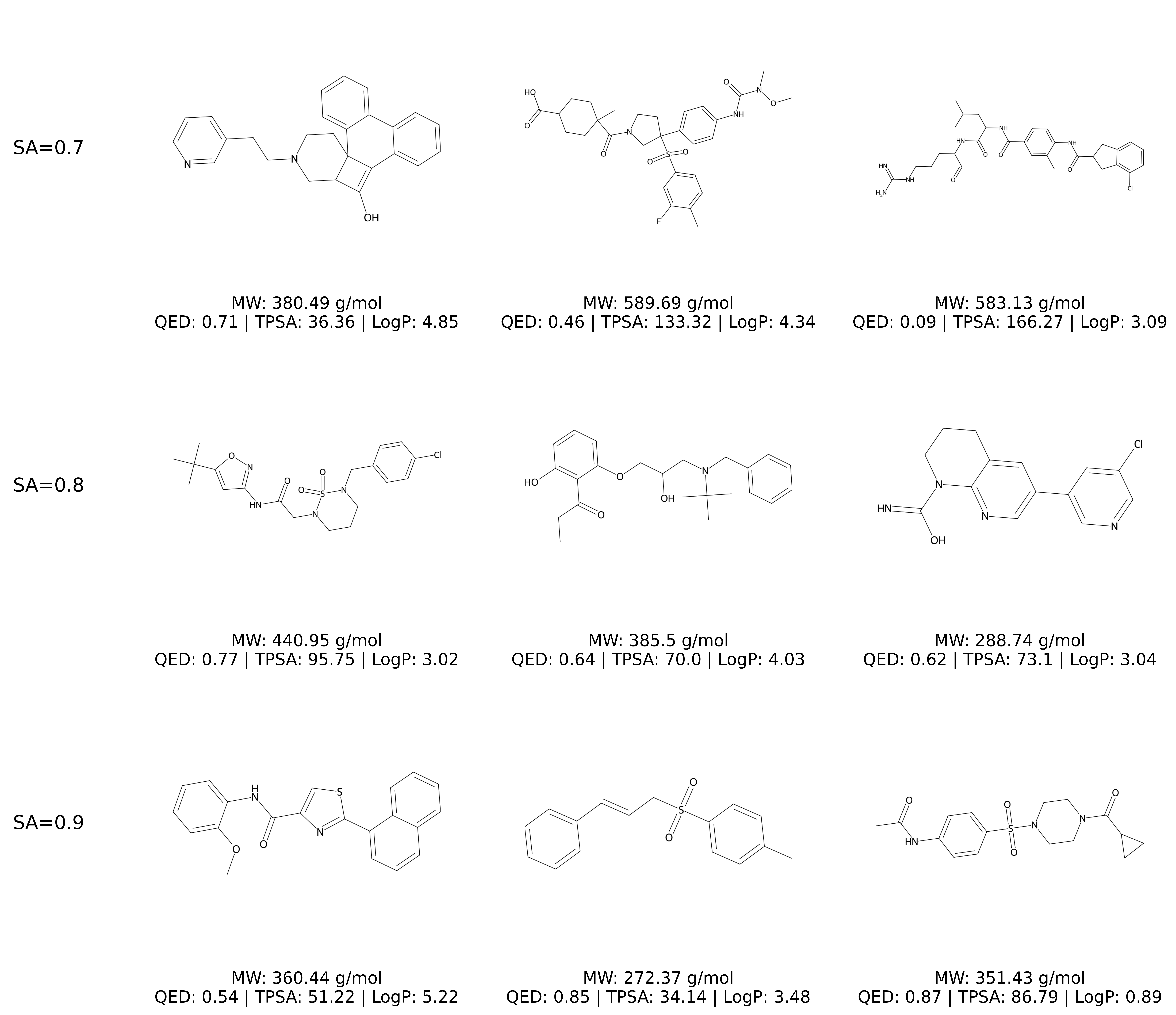}
   \caption{Structures of molecules generated by Hyformer conditioned on SA score, visualized using RDKit, along with their chemical properties.
   }
   \label{fig:generatedmoleculesconditionedSA}
\end{figure}

\begin{figure}[!h]
    \centering
    \includegraphics[width=\columnwidth]{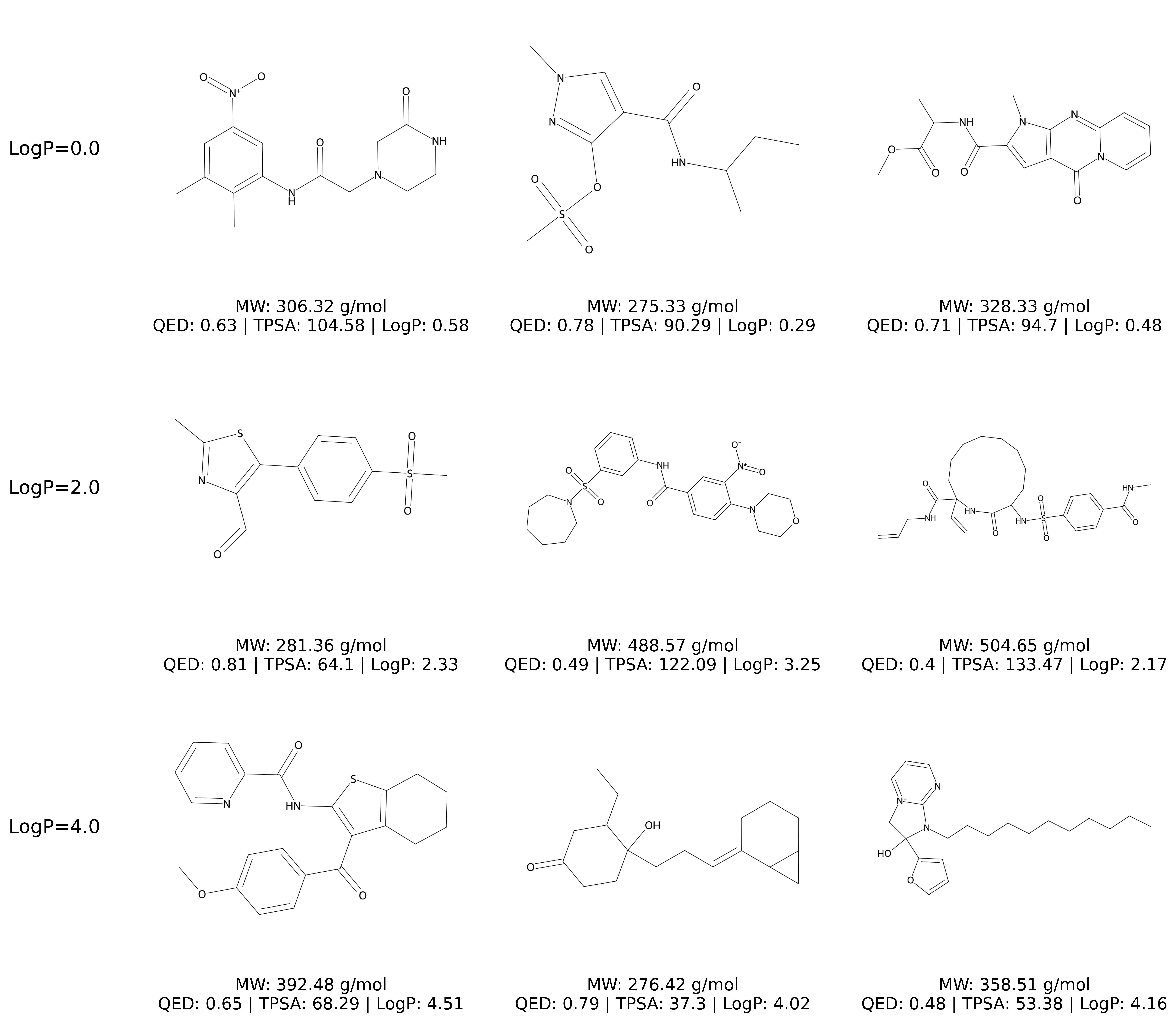}
    \caption{Structures of molecules generated by Hyformer conditioned on LogP values, visualized using RDKit, along with their chemical properties.
    }
    \label{fig:generatedmoleculesconditionedlogP}
\end{figure}

\subsection{Qualitative Evaluation of Learned Representations}

We next examine the Hyformer embeddings in the context of the chemical properties of the molecules (Fig.~\ref{fig:hyformer_2d_embeddings}). To this end, we randomly sample 20,000 molecules and pass them through $\HyF$’s encoder, pre-trained for molecular property prediction in Section~\ref{section:molecular-property-prediction}, to obtain molecule embeddings. We visualize the embeddings in two dimensions through principal components analysis (PCA) and color them according to their four chosen chemical properties (LogP, TPSA, QES, MW).

Qualitatively, the spatial arrangement of molecules is clearly connected to their chemical properties. Furthermore, embeddings exhibit a smooth profile of change w.r.t.\ each property. These observations indicate that $\HyF$ learns well-behaved, information-rich molecular representations.

\begin{figure}[!ht]
\begin{center}
\centerline{\includegraphics[width=\columnwidth]{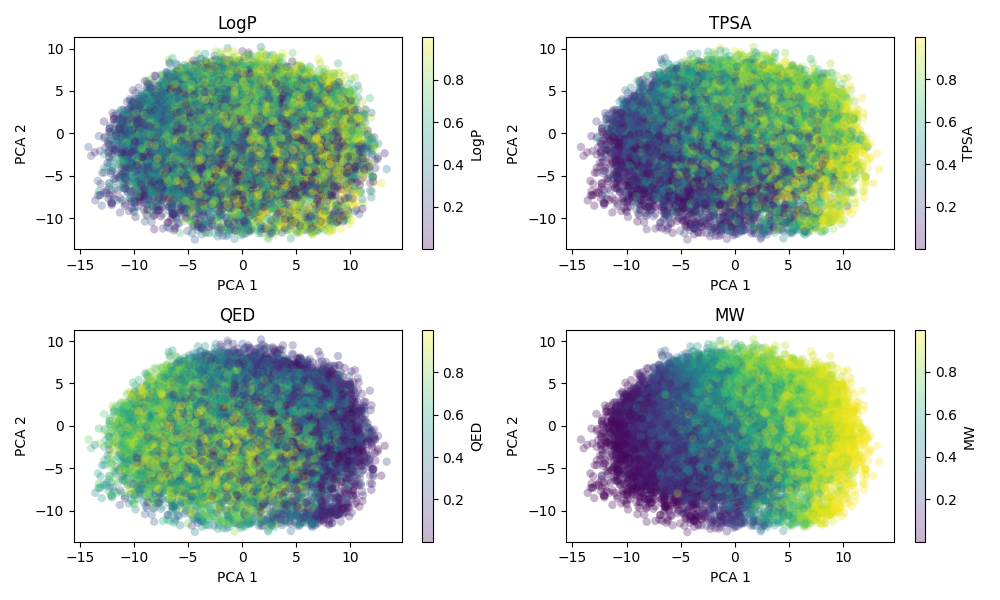}}
\caption{Hyformer's molecular embeddings. The considered chemical properties are normalized to lie in the $[0, 1]$ interval.}
\label{fig:hyformer_2d_embeddings}
\end{center}
\vskip -0.3in
\end{figure}

\end{document}